\documentclass[twoside]{article}

\usepackage[accepted]{aistats2023}
\usepackage[round]{natbib}

\usepackage{algorithm2e}
\usepackage{amssymb}

\usepackage{amsmath}
\usepackage{amsthm}
\usepackage{mathtools}
\usepackage{thmtools} 
\usepackage{thm-restate}
\usepackage{placeins}

\usepackage{subcaption}

\usepackage{svg}
\usepackage{hyperref}
\hypersetup{
    citecolor=black,
    colorlinks=true,
    linkcolor=black,
    filecolor=magenta,      
    urlcolor=cyan
    }

\newtheorem{lemma}{Lemma}

\newtheorem{assumption}{Assumption}

\theoremstyle{definition}
\newtheorem{example}{Example}

\newtheorem{remark}{Remark}

\DeclareMathOperator*{\argmax}{arg\,max}
\DeclareMathOperator*{\argmin}{arg\,min}

\newcommand{\norm}[1]{\lVert#1\rVert}

\newcommand{\R}{\mathbb{R}}
\newcommand{\calA}{\mathcal{A}}
\newcommand{\calD}{\mathcal{D}}
\newcommand{\calE}{\mathcal{E}}

\newcommand{\calG}{\mathcal{G}}

\newcommand{\calO}{\mathcal{O}}
\newcommand{\calP}{\mathcal{P}}
\newcommand{\calR}{\mathcal{R}}
\newcommand{\calS}{\mathcal{S}}
\newcommand{\calT}{\mathcal{T}}
\newcommand{\calX}{\mathcal{X}}

\newcommand{\starf}{f^*}
\newcommand{\starg}{g^*}
\newcommand{\starh}{h^*}

\newcommand{\starpi}{\pi^*}
\newcommand{\starcalG}{\mathcal{G}^*}

\newcommand{\hata}{\hat{a}}

\newcommand{\hatg}{\hat{g}}
\newcommand{\hatm}{\hat{m}}

\newcommand{\hatmu}{\hat{\mu}}
\newcommand{\hatp}{\hat{p}}

\newcommand{\safe}{\textit{safe}}
\newcommand{\true}{\text{True}}

\newcommand{\IGW}{\texttt{IGW}}
\newcommand{\HIGW}{\texttt{HTE-IGW}}
\newcommand{\MIGW}{\texttt{MOD-IGW}}
\newcommand{\MHIGW}{\texttt{MOD-HTE-IGW}}

\newcommand{\CheckAndChooseSafe}{\text{CheckAndChooseSafe}}

\newcommand{\msafe}{m^*}
\newcommand{\msafealg}{\hat{m}}

\newcommand{\misspecError}{B}

\DeclareMathOperator{\Uniform}{Uniform}
\newcommand{\E}{\mathop{\mathbb{E}}}

\newcommand{\Reg}{\text{Reg}}

\newcommand{\A}{\mathcal{A}}

\newcommand{\Xscript}{\mathcal{X}}
\newcommand{\ordO}{\mathcal{O}}

\newcommand{\I}{\mathbb{I}}

\newcommand{\G}{\mathcal{G}}

\newcommand{\eventReg}{\mathcal{W}_1}

\newcommand{\N}{\mathbb{N}}

\newcommand\blfootnote[1]{%
  \begingroup
  \renewcommand\thefootnote{}\footnote{#1}%
  \addtocounter{footnote}{-1}%
  \endgroup
}

\begin{document}

%

%

\twocolumn[

\aistatstitle{Flexible and Efficient Contextual Bandits with Heterogeneous Treatment Effect Oracles}

\aistatsauthor{ Aldo Gael Carranza* \And Sanath Kumar Krishnamurthy* \And Susan Athey }

\aistatsaddress{ Stanford University \And  Stanford University \And Stanford University } ]

\begin{abstract}
  Contextual bandit algorithms often estimate reward models to inform decision-making. However, true rewards can contain action-independent redundancies that are not relevant for decision-making. We show it is more data-efficient to estimate any function that explains the reward differences between actions, that is, the \textit{treatment effects}. Motivated by this observation, building on recent work on oracle-based bandit algorithms, we provide the first reduction of contextual bandits to general-purpose heterogeneous treatment effect estimation, and we design a simple and computationally efficient algorithm based on this reduction. Our theoretical and experimental results demonstrate that heterogeneous treatment effect estimation in contextual bandits offers practical advantages over reward estimation, including more efficient model estimation and greater flexibility to model misspecification. 
\end{abstract}

\section{INTRODUCTION}\label{sec:intro}

We study cumulative regret minimization in the stochastic contextual bandit setting \citep{li2010contextual}. A common approach in contextual bandits is to use a specified function class to estimate a model of the expected reward conditional on any given context-action pair and utilize the estimated model to recommend actions \citep{chu2011contextual, agrawal2013thompson, foster2020beyond, simchi2021bypassing}. Such a model estimate is always biased when the given model class does not capture the true conditional expected reward (e.g., a linear model class is used in reward estimation when the true reward model is nonlinear). This bias results in an unavoidable additional linear regret term, which is asymptotically dominant, in the corresponding contextual bandit algorithm \citep{krishnamurthy2021adapting, foster2020adapting}. 

In this paper, we obtain tighter cumulative regret bounds by reducing the linear term due to model bias by shifting the target of model estimation for decision-making. We observe that the conditional expected reward does not need to be entirely estimated to make decisions. It suffices to estimate any function that explains the reward differences between actions, that is, any model that captures \textit{treatment effects}. We show the bias of treatment effect estimation is no greater (and often significantly smaller) than the bias of a reward model estimation under the same model class. An example of lower bias in treatment effect estimation for a fixed model class is the case where the true reward is the sum of a non-linear function of only contexts and a linear function of contexts and actions. In this case, while the conditional reward functions may require complex model classes to be well approximated, a linear function class can effectively capture the true treatment effects.

Motivated by these observations, we design an improved contextual bandit algorithm that relies on an oracle for heterogeneous treatment effect (HTE) estimation known as the $R$-learner \citep{nie2021quasi}. The $R$-learner possesses certain properties that make it more useful as a heterogeneous treatment effect estimator for model estimation in contextual bandits compared to other comparable methods, including $S$-learner, $T$-learner, or $X$-learner \citep{kunzel2019metalearners}. The principal reason is that, unlike other methods which attempt to form uniformly accurate estimates of treatment effects over the context distribution, the $R$-learner estimates treatment effects more accurately in regions of exploration in the context space where actions are sampled non-deterministically. Since contextual bandits explore more in regions of context space where there is more uncertainty regarding the best action, the $R$-learner's more accurate treatment effect estimates in these regions allow our algorithms to achieve optimal regret guarantees. \blfootnote{*Equal contribution} 

In this paper, we provide the first general-purpose reduction of contextual bandits to heterogeneous treatment effect estimation via $R$-loss minimization. Using this construction, we design a computationally efficient algorithm for contextual bandits with an $R$-loss oracle. We show that our approach is more robust to model misspecification than reward estimation methods based on squared error regression oracles. Experimentally, extending our approach to consider model selection, we validate that our algorithm outperforms these regression-based methods in settings where the treatment effect lies in a simpler class of functions than that of the reward.

\subsection{Related Work}

\paragraph{Contextual Bandits}
Contextual bandit algorithms can generally be classified into two categories of realizability-based or agnostic approaches. Realizability-based methods \citep{filippi2010parametric, abbasi2011improved, chu2011contextual, agrawal2013thompson, foster2018practical, foster2020beyond, simchi2021bypassing} are computationally simple yet require stricter assumptions on the underlying reward model, and agnostic methods \citep{auer2002nonstochastic, langford2007epoch, dudik2011efficient, beygelzimer2011contextual, agarwal2014taming} have weaker assumptions but are generally computationally intractable \citep{bietti2021contextual}.
Our approach attempts to achieve a middle ground that captures the simplicity of realizability-based methods and robustness of agnostic methods.
Our work extends a recent line of research on contextual bandit algorithms with least squares regression oracles \citep{foster2020beyond, simchi2021bypassing} that utilize the inverse gap weighting action-sampling strategy of \citet{abe1999associative}. In contrast to this work, we develop a more flexible model estimation strategy built on a heterogeneous treatment effect estimation oracle primitive. \citet{foster2020adapting} and \citet{krishnamurthy2021adapting} extend the least squares regression oracle-based approach to the setting where reward model realizability may not hold and the misspecification error is unknown. Our approach relies on the misspecification adaptation strategies presented in these works.

Similar to our approach, \citet{greenewald2017action, krishnamurthy2018semiparametric, peng2019practical, choi2022semiparametric} deal with estimating the essential components of a reward function for decision making. Although \citet{krishnamurthy2018semiparametric} use a similar local centering transformation as the $R$-learner to estimate treatment effects, their UCB-style approach is only applicable for linear treatment effect function classes. Similarly, \citet{greenewald2017action,peng2019practical, choi2022semiparametric} consider UCB-based or Thompson sampling-based approaches to estimate a partially linear model. In contrast, our approach is flexible for any specified class of treatment effect models, not just linear models.

Various works have taken a similar yet distinct approach on expanding reward estimation in contextual bandits beyond the standard setting.
The problem of non-stationary confounding in contextual bandits is considered in \citet{qin2022adaptivity, choi2022semiparametric}, it is addressed using a Thompson sampling-based approach.
Our algorithm provides an oracle-based approach that is capable of handling similar issues of non-stationary confounding without the restrictive assumptions of Thompson-sampling (see Section \ref{sec:extensions}). \citet{hu2020smooth, gur2022smoothness} consider contextual bandits with non-parametric reward functions with unknown smoothness which allows for more flexible model estimation in a general class of problems. Our work is supplemented by this work on non-parametric contextual bandits since it is also capable of handling non-parametric model classes for estimating the underlying treatment effect model. Contextual dueling bandits \citep{dudik2015contextual,saha2022efficient} consider the problem where the learner only observes the reward differences between pairs of actions. In contrast, although our approach captures residual reward differences from the mean, we are still in the standard setting where reward outcomes from individually selected arms are observed.

\paragraph{Heterogeneous Treatment Effect Estimation}
There are plenty of heterogeneous treatment effect estimation methods in the literature including the $S$-learner, $T$-learner, $X$-learner, and DR-learner \citep{kunzel2019metalearners, kennedy2020optimal}. As discussed in the introduction, we found the $R$-learner satisfies certain properties that make it a more natural fit for oracle-based contextual bandits. The $R$-learner was proposed in its full generality by \citet{nie2021quasi} as a flexible and general-purpose method for heterogeneous treatment effect estimation amenable to machine learning methods since it is formulated as a risk minimization procedure over a specified class of treatment effect models. It is based on Robinson's local centering transformation \citep{robinson1988root} for estimating parametric components of partially linear models, which recent methods have built on for more general semi-parametric inference of treatment effects \citep{athey2019generalized, chernozhukov2018double}.

The exploration parameters of our algorithm rely on knowledge of excess risk bounds on the $R$-learner. \citet{nie2021quasi} and \citet{foster2019orthogonal} provide excess risk bounds on the $R$-learner which are not applicable to our setting. We look towards existing results on problem-dependent error bounds in \cite{koltchinskii2011oracle,xu2020towards} that we show are applicable for establishing finite-sample excess risk bounds for the multi-treatment $R$-learner with known propensities.

\section{PRELIMINARIES}\label{sec:prelim}


\paragraph{Learning Setting}
Let $\calA$ be a set of $K$ actions and $\calX$ be a space of contexts. A sequence of interactions between a learner and nature occurs over $T$ rounds, where $T$ is possibly unknown. At each round $t$, nature randomly samples a context-reward pair $(x_t, r_t)$ from a fixed but unknown probability distribution $D$ over the joint space of contexts and reward vectors $\calX\times[0,1]^K$. The learner observes $x_t$, chooses an action $a_t\in\calA$, and receives the reward $r_t(a_t)\in[0,1]$ for the chosen action. Note that the rewards are assumed to be within $[0,1]$ for simplicity, but our results can easily be extended more generally to bounded rewards.

Define the true conditional mean reward function to be $\starf(x,a)\coloneqq\E[r(a)|x]$ for any $(x,a)\in\calX\times\calA$, where the expectation is taken with respect to $D$. The policy that maximizes the true conditional mean reward is $\starpi(x)\coloneqq\argmax_a \starf(x,a)$. The goal of the learner is compete with the optimal policy $\starpi$ to minimize cumulative regret:
\begin{equation} \label{eq:cum_regret}
    \text{Reg}(T)\coloneqq\sum_{t=1}^T\big(r_t(\pi^*(x_t))- r_t(a_t)\big).
\end{equation}

\paragraph{Reward Decomposition}
Note that, for any function $\starh$ of only contexts, there is a corresponding function $\starg$ of contexts and actions such that
\begin{equation}
    \starf(x,a)=\starg(x,a)+\starh(x)
\end{equation}
for all $(x,a)\in\calX\times\calA$. Under any such decomposition of the conditional expected reward, we will refer to $\starh$ as a \textit{confounder} model and $\starg$ as the corresponding \textit{true treatment effect} model. A true treatment effect model is named as such because it sufficiently captures the conditional reward differences between actions, that is, given any $x\in\calX$ we have that $\starg(x,a)-\starg(x,a')=\starf(x,a)-\starf(x,a')$ for any $a,a'\in\calA$.

Note that such a decomposition of the conditional reward function is not unique and, in fact, can be trivial, i.e., $\starh\equiv0$. However, if there exists a decomposition such that $\starg$ can be captured by a ``simpler" model class than that required to capture $\starf$, the learner could more efficiently estimate $\starg$ instead for decision-making.
This would be sufficient since the optimal policy can be equivalently defined as $\starpi(x)=\argmax_a\starg(x,a)$.

We assume that the learner has access to a class $\calG$ of measurable functions $g:\calX\times\calA\to[0,1]$ that models a true treatment effect. These functions are referred to as \textit{treatment effect} models. We do not require treatment effect model realizability, i.e., there does not necessarily exist a model $g\in\calG$ such that $g(x,a)=\starg(x,a)$ for all $(x,a)\in\calX\times\calA$. We discuss treatment effect model misspecification in Section \ref{sec:rlearner-oracle}.

\paragraph{Additional Notation}
The probability simplex over actions $\calA$ is denoted as $\Delta_\calA\coloneqq\{q\in[0,1]^K|\sum_{a\in\calA}q_a=1\}$. For any probability kernel $p:\calX\to\Delta_\calA$ in the set of all probability kernels, denoted as $\calP$, we let $p(\cdot|x)$ denote the associated probability distribution over $\calA$ for a given context $x\in\calX$. We denote $D(p)$ to be the observable data-generating distribution over $\calX\times\A\times[0,1]$ induced by $p\in\calP$ on the distribution $D$ such that sampling $(x,a,r(a))\sim D(p)$ is defined as sampling $(x,r)\sim D$ then sampling $a\sim p(\cdot|x)$ and only revealing $(x,a,r(a))$. We let $D_{\calX}$ denote the marginal distribution of $D$ on the set of contexts $\calX$. With a slight abuse of notation, for any function $g:\calX\times\calA\to[0,1]$ and context $x\in\calX$, we let $g(x,\cdot)$ denote the vector $(g(x,a))_{a\in\calA}$. Let $\langle\cdot,\cdot\rangle$ denote the inner product operator.

\section{HETEROGENEOUS TREATMENT EFFECT ORACLE} \label{sec:rlearner-oracle}

We will use the $R$-learner, a method for heterogeneous treatment effect estimation based on $R$-loss minimization \citep{nie2021quasi}, as the model fitting subroutine in our contextual bandit algorithm. In this section, we describe the multi-treatment $R$-loss and the $R$-learner model fitting oracle we assume access to. We also discuss the benefits of the $R$-learner for contextual bandits.

\subsection{$R$-loss Risk, Excess Risk, and Misspecification Error}\label{sec:r-loss-risk-error}

Let $p:\calX\to\Delta_\calA$ be a fixed probability kernel and let $e_a\in\{0,1\}^K$ be the one-hot encoding of action $a\in\calA$. The multi-treatment $R$-loss \textit{risk} of a model $g\in\calG$ under kernel $p$ is
\begin{equation}\label{eq:r-loss-risk}
    \calR_p(g)\coloneqq\E\big[\big(r(a)-\hatmu(x)-\big\langle e_{a}-p(x),g(x, \cdot)\big\rangle\big)^2\big],
\end{equation}
where the expectation is taken with respect to $D(p)$ and $\hatmu$ is an estimate of the conditional mean of rewards realized under the action-sampling kernel $p$, defined as $\mu(x)\coloneqq\E_{a\sim p(\cdot|x)}[\starf(x,a)]$ for all $x\in\calX$.\footnote{We use an estimate of $\mu$ in our definition of $R$-loss risk since it is unknown in our problem setting. The action-sampling probabilities under kernel $p$, on the other hand, are known.}

The $R$-loss \textit{excess risk} of a model $g\in\calG$ under kernel $p$ is
\begin{equation}\label{eq:excess-risk}
    \calE_p(g)\coloneqq\calR_p(g)-\min_{g'}\calR_p(g'),
\end{equation}
where the infimum is taken with respect to all functions.
The $R$-loss excess risk has an important role in our analysis since it is fundamental to the evaluation of treatment effect model quality.

The \textit{average misspecification error} of the model class $\calG$ under the $R$-loss is defined as the quantity $\sqrt{B}$,\footnote{Similar measures of misspecification are denoted by $\epsilon$ in other works on misspecified models \citep{ghosh2017misspecified, foster2020beyond, foster2020adapting, lattimore2020learning}.} where
\begin{equation}\label{eq:misspecification-error}
    \misspecError\coloneqq\max_{p\in\calP}\min_{g\in\calG}\calE_p(g).
\end{equation}
Capturing this quantity of model misspecification is useful for doing away with an assumption on model realizability.

\subsection{$R$-loss Properties}
The $R$-loss is based on Robinson's local centering transformation\footnote{The multivariate version of Robinson's transformation is $r(a)-\mu(x)=\langle e_a-p(x), \starg(x,\cdot)\rangle+\varepsilon$ where $\varepsilon$ is a noise term such that $\E[\varepsilon|x,a]=0$.} and was introduced as a minimization objective for flexible heterogeneous treatment effect estimation \citep{robinson1988root, nie2021quasi}.
Indeed, the infima of the $R$-loss risk correspond exactly with the true treatment effect models (see Appendix \ref{app:excess-risk-id} for further discussion).

The nuisance parameters $\mu$ and $p$ of the $R$-loss are typically estimated from data. In our setting, the exact treatment propensities $p(a|x)$ are known to the bandit learner because they are exactly the action-sampling probabilities of the bandit algorithm and these quantities suffice for model estimation. However, the conditional mean realized rewards $\mu(x)$ are not known, so $\mu$ must be estimated. Even so, the impact of this nuisance parameter estimation is almost negligible to our analysis as we show in the following proposition that the $R$-loss achieves a form of robustness to unknown nuisance parameter $\mu$ in the sense that the excess $R$-loss risk does not depend on any fixed estimate of $\mu$ when $p$ is known.
See Appendix \ref{app:excess-risk-id} for the proof.
\begin{restatable}{proposition}{lemExcessRisk} \label{lem:r-excess-risk}
    Suppose $\hatmu$ is any fixed function. For any $\starg\in\argmin_{g'}\calR_p(g')$,
    \begin{equation}\label{eq:excess-risk-id}
        \calE_p(g)=\E\left[\big\langle e_{a}-p(x),g(x,\cdot)-\starg(x,\cdot)\big\rangle^2\right],
    \end{equation}
    where the expectation is taken with respect to $D(p)$.
\end{restatable}
This result holds for any given model $g\in\calG$, and so any uniform estimate $\hatmu$ suffices (e.g., $\hatmu\equiv0$). However, a more accurate estimate of $\mu$ reduces the variance of the $R$-loss,\footnote{In empirical risk minimization, excess risk rates depend on both the variance of the loss being minimized and the complexity of the hypothesis class. Due to randomness in sampled data, loss variance is typically non-zero.} thus improving the finite-sample excess risk rates of estimating the treatment effect model by a problem-dependent constant factor. Therefore, in our implementation, we choose to train $\hatmu$ on available data using a \textit{cross-fitting} strategy where the function is estimated on different data folds and the output $\hatmu(x)$ for a given point $x$ is the output under the function estimated on the fold corresponding to $x$. This ensures a better estimate of the nuisance parameter that is still fixed relative to the sampled data, as suggested in \cite{chernozhukov2018double, nie2021quasi}.

Note that Equation \ref{eq:excess-risk-id} vanishes when the action-sampling kernel is deterministic. Therefore, we observe that Proposition \ref{lem:r-excess-risk} reveals that the $R$-loss excess risk does not penalize models with less accurate treatment effect estimates in regions where actions are sampled deterministically. This allows capturing a more precise form of uncertainty quantification useful for proving tighter regret bounds. In particular, this allows us to prove the following proposition showing that, for a given model class, the average misspecification error under the $R$-loss is bounded above by a corresponding notion of average misspecification error under the \textit{squared error loss} typically used in reward model estimation \citep{simchi2021bypassing}.
\begin{restatable}{proposition}{lemMisspecificationBound} \label{lem:misspecification-error-bound}
    For a fixed model class $\calG$, the average misspecification error under the $R$-loss is no greater than the average misspecification error under the squared error loss, i.e.,
    \begin{equation}\label{eq:misspecification-error-bound}
        B\le\max_{p\in\calP}\min_{f\in\calG}\E\left[\big(f(x,a)-\starf(x,a)\big)^2\right].
    \end{equation}
    where the expectation is taken with respect to $D(p)$.
\end{restatable}
This result will be important in highlighting the benefits of treatment effect estimation over reward estimation in ensuring tighter regret bounds under model misspecification.
See Appendix \ref{app:misspecification-bound} for a proof and further discussion on the intuition of this result. 


\subsection{$R$-loss Oracle}\label{sec:oracle}

We will assume access to an oracle that minimizes the $R$-loss risk given finite data samples such that the excess risk of the oracle output can be bounded by a known estimation error rate on the given model class $\calG$ and the misspecification error $\misspecError$. We present this formally in the following assumption.
\begin{assumption}\label{ass:main-assumption}
    Let $\calD$ be a data set of $n$ independently and identically drawn samples from $D(p)$ under some fixed kernel $p$.
    Let $\hatg$ be the output model estimate of a given $R$-loss oracle with $\calD$ as input. For any $\zeta\in(0,1)$, the $R$-loss excess risk of $\hatg$ is bounded as
    \begin{equation}
    \label{eq:regression-rate}
        \calE_p(\hatg)\le \xi(n,\zeta)+C_0\misspecError,
    \end{equation}
    with probability $1-\zeta$, where $C_0$ is a universal constant and $\xi(n,\zeta)$ is the known estimation rate for the model class $\calG$ that is non-increasing in $n\in\N$ for all $\zeta\in(0,1)$.
\end{assumption}

An example of an estimator that satisfies Assumption \ref{ass:main-assumption} is the \textit{$R$-learner} which performs empirical $R$-loss risk minimization over the specified class of models $\calG$:
\begin{equation}
    \hatg\in\argmin_{g\in\calG}\widehat{\calR}_p(g;\calD)+\Lambda(g),
\end{equation}
where $\calD=\{(x_t,a_t,r_t(a_t))\}_{t\in\calT}$ is the given input dataset indexed by $\calT$ of samples drawn i.i.d.~from $D(p)$, $\Lambda$ is a regularizer on the complexity of the treatment effect model, and
\begin{equation}
    \widehat{\calR}_p(g;\calD)\!=\!\frac{1}{|\calT|}\!\sum_{t\in\calT}\!\Big(\!r_t(a_t)-\hat\mu_t(x_t)-\big\langle\! e_{a_t}\!-p(x_t),g(x_t,\cdot)\!\big\rangle\!\Big)^{\!2}
\end{equation}
is the empirical $R$-loss risk which uses an estimate $\hatmu_t$ for $\mu$ that is constructed to be independent of $(x_t,a_t,r_t(a_t))$ for each $t\in\calT$, typically via cross-fitting as previously discussed.

In Appendix \ref{app:r-learner-estimation-rate}, we discuss valid finite-sample $R$-loss excess risk rates for the $R$-learner under various important classes of models including classes with finite VC dimension and non-parametric classes of polynomial growth.\footnote{Since excess risk bounds for empirical risk minimizers can be bounded by the complexity of the model class, Assumption \ref{ass:main-assumption} can be proved for several classes using results in \cite{koltchinskii2011oracle}.}

\begin{remark}
Note that the oracle assumption as stated requires i.i.d.~data from the data-generating process induced by a fixed kernel. In our bandit setting where model estimates inform kernel updates, this requirement would seem to imply that model estimation must be done only using data collected with the most recent kernel and any update on the kernel requires a subsequent refresh on the data buffer, reducing the effective sample size. However, this i.i.d.~requirement is purely due to technical reasons and we can do away with it using martingale arguments that extend our analysis to adaptive kernels. Nevertheless, we keep the i.i.d.~assumption in this paper to maintain focus on our contributions to more efficient model estimation via treatment effect estimation and to avoid a more complicated discussion of martingales.
\end{remark}

\subsection{$R$-loss Benefits for Contextual Bandits}
We elaborate on the benefits of $R$-learner over other heterogeneous treatment effect estimators for the contextual bandit setting. For simplicity, we restrict our discussion to the two arm setting $\calA=\{a_0,a_1\}$.
This is only a simplification for the purposes of this discussion as our algorithm allows for more than two arms.

In this setting, estimating the true HTE function defined by $\tau^*(x):=f^*(x,a_1)-f^*(x,a_0)$ for all $x\in\calX$ suffices for decision-making in contextual bandits.
Most HTE estimators are designed to estimate a function $\hat{\tau}$ such that they minimize $\E_x[(\hat{\tau}(x)-\tau^*(x))^2]$. As one may expect, at any context $x$, the error $(\hat{\tau}(x)-\tau^*(x))^2$ in estimating treatment effects is likely to get large as $\min(p(a_0|x),p(a_1|x))$ gets closer to zero (i.e., less experimentation at context $x$). For example, the weights used by DR-learner \citep{kennedy2020optimal} grow as action sampling probabilities shrink, leading to higher variance for this estimator. This implies that most HTE estimators suffer slow estimation rates for the contextual bandit setting, where action sampling probabilities need to shrink in order to obtain optimal cumulative regret guarantees. 

We argue that our analysis of the $R$-learner allows us to side step this issue (when assignment probabilities are known)
without assuming that assignment probabilities are bounded away from zero. In particular, one can easily show from Proposition \ref{lem:r-excess-risk} that the $R$-learner excess risk in the two arm setting is given by $\E_x[p(a_0|x)p(a_1|x)(\hat{\tau}(x)-\tau^*(x))^2]$. Unlike the previous estimand, note that $p(a_0|x)p(a_1|x)(\hat{\tau}(x)-\tau^*(x))^2$ goes to zero as the action sampling probability for either treatment or control goes to zero. Therefore, the $R$-loss excess risk can be bounded by fast rates even with arbitrarily small action-sampling probabilities (when they are known).
This insight extends to the general multi-arm setting.

Ultimately, we show the $R$-loss excess risk is exactly the quantity we need to bound in order to obtain our optimal regret guarantees (see Lemma \ref{lem:reg-est-accuracy}).
Therefore, the additional properties we described in this section, of fast rates for the $R$-loss excess risk under vanishing action-sampling probabilities and tighter control on model misspecification error (stated in Proposition \ref{lem:misspecification-error-bound}), lead us to identify the $R$-learner as an ideal HTE estimator for the contextual bandit setting.

\section{ALGORITHM}\label{sec:algorithm}

\SetAlgoVlined
\RestyleAlgo{ruled}
\begin{algorithm}
    \caption{Heterogeneous Treatment Effect Inverse Gap Weighting}
    \label{alg:safe-falcon}
    \KwIn{epoch schedule $\tau_m$ for all $m\in\N$, confidence parameter $\delta\!\in\!(0,1)$}
    
    Initialize $\gamma_1=1$, $\hatg_1\equiv 0$, $\tau_0=0$, $H_0=\varnothing$, and $\safe=\true$.

    \For{epoch\, $m=1,2,\dots$}{
        
        \For{round $t=\tau_{m-1}+1,\dots, \tau_{m}$ }{ 
        
            Observe context $x_t$.
             
            \eIf{$\safe$}{ 
                Let kernel $p_m$ be defined by \eqref{eq:action_kernel} given $\gamma_m$ and $\hatg_m$.
                
                Sample $a_t \sim p_m(\cdot|x_t)$ and observe $r_t(a_t)$.
                
                Update $H_t = H_{t-1}\cup\{(x_t,a_t,r_t(a_t)\}$.
                
                Update $(\safe, \hatm) \leftarrow \CheckAndChooseSafe(H_{t})$.
            }{
                Sample $a_t\sim p_{\msafealg}(\cdot|x_t)$ and observe $r_t(a_t)$.
            }

        }
        
        \If{$\safe$}{
            Let $\gamma_{m+1}=c\sqrt{K/\xi(\tau_{m}-\tau_{m-1},\delta'/(m+1)^2)}$.
            
            Let $\hatg_{m+1}$ be the $R$-loss oracle output with input data $H_{\tau_{m}}\!\backslash H_{\tau_{m-1}}$ from epoch $m$.
        }
    }
\end{algorithm}

In this section, we outline the Heterogeneous Treatment Effect Inverse Gap Weighting ($\HIGW$) algorithm for contextual bandits. A formal description is stated in Algorithm \ref{alg:safe-falcon} below. Our algorithm is based on the \texttt{FALCON+} algorithm in \citep[Algorithm 2]{simchi2021bypassing} and the \texttt{Safe-FALCON} algorithm in \citep[Algorithm 1]{krishnamurthy2021adapting}.

\paragraph{Epoch Schedule}
$\HIGW$ is implemented using a schedule of epochs where epoch $m\in\N$ starts at round $\tau_{m-1}+1$ and ends at round $\tau_m$. This is done to reduce the number of calls to the $R$-loss oracle. For example, if we set $\tau_m=2^m$, then the algorithm runs in $\calO(\log T)$ epochs for any (possibly unknown) $T$. The algorithm accepts any valid epoch schedule. For any round $t\in\N$, we let $m(t)$ denote the epoch corresponding to round $t$.

\paragraph{Adaptive Behavior}
Each round $t$ starts with a status that is called ``safe'' or ``not safe'', depending on whether the algorithm has detected evidence of model misspecification. At the beginning of any epoch $m$ that starts on a ``safe'' round, the algorithm uses an estimate of the treatment effect model $\hatg_m$, obtained via an offline $R$-loss oracle using data from the previous epoch $m-1$. As long as the ``safe'' status is maintained, the algorithm assigns actions by drawing from the following \emph{action selection kernel}:
    \begin{align} \label{eq:action_kernel}
        p_m(a|x)\coloneqq
        \begin{cases}
            \frac{1}{K+\gamma_m \left(\hatg_m(x, \hat{a}) - \hatg_m(x,a) \right)} &\text{for } a\neq \hat{a} \\
            1 - \sum_{a'\neq \hat{a}} p(a'|x) &\text{for } a=\hat{a},
        \end{cases}
    \end{align}
where $\hat{a} = \argmax_a \hatg_m(a, x)$ is the current predicted best action. The parameter $\gamma_m > 0$ governs how much the algorithm exploits and explores: assignment probabilities concentrate on the predicted best policy $\hat{a}$ when $\gamma_m$ is large and are more spread out when $\gamma_m$ is small. During ``safe'' epochs, this parameter is set to be
\begin{align} \label{eq:gamma}
    \gamma_m := c\sqrt{K/\xi(\tau_{m-1} - \tau_{m-2}, \delta'/m^2)},
\end{align}
where $c=\sqrt{1/8}$ is a fixed constant, $\xi$ is the estimation rate of the $R$-loss oracle defined in \eqref{eq:regression-rate}, $\delta'=\delta/2$ where $\delta > 0$ is a confidence parameter, and the quantity $\tau_{m-1}-\tau_{m-2}$ is the size of the previous batch which was used to estimate the model $\hatg_m$. Definition \eqref{eq:gamma} implies that small classes such as linear models allow for a quickly increasing $\gamma_m$ and therefore more exploitation, while large classes require more exploration since $\gamma_m$ increases more slowly.

If the algorithm ever switches to ``not safe'' status, it switches to a safe kernel that was previously used and never returns to ``safe'' status.
This safe kernel guarantees a bounded level of regret. To determine the safety status of the bandit, we use a misspecification test.

\paragraph{Misspecification Test}
We make use of the same misspecification testing procedure developed in \citet{krishnamurthy2021adapting} to determine the safety status of every round. This misspecification test constructs a Hoeffding-style high-probability lower bound on the cumulative reward and triggers the safety check mechanism when the cumulative reward dips below this lower bound. If triggered, the safety mechanism defaults to the most recent kernel with the largest lower bound. \citet{krishnamurthy2021adapting} show that the safety test is triggered only after the lower bounds on rewards are sufficiently large. Therefore, this high-probability safety test ensures that the expected regret stays bounded as long as the safety test is not triggered.

In our algorithm, this misspecification test is implemented as the $\CheckAndChooseSafe$ method which takes in the interaction history to return an indicator variable $\safe$, indicating whether the misspecification test passed, and the safety epoch $\hatm$, pointing to the previously used safe kernel $p_{\hatm}$. Once the algorithm switches to ``not safe'' status, it continues to use the safe kernel and never returns to ``safe'' status. The details of an efficient implementation of this misspecification test can be viewed in Algorithms 2 and 3 of \citet{krishnamurthy2021adapting}. For our purposes, it suffices to know that this test guarantees bounds on the instantaneous regret after the misspecification error surpasses the estimation error. Although the results are important, the details of the implementation are somewhat orthogonal to the emphasis of this paper.
See Appendix \ref{app:safety-check-assumptions} for further discussion.
Lastly, note that the misspecification level need not be known, though our regret bounds stated in Section \ref{sec:main_results} will depend on it.

\section{MAIN RESULTS}\label{sec:main_results}

The following provides a regret guarantee on our algorithm that depends on the known estimation rates and the misspecification error of the $R$-loss oracle.

\begin{restatable}[]{theorem}{thmMain} \label{thm:main-theorem}
Suppose the $R$-loss oracle satisfies Assumption \ref{ass:main-assumption}. Then, with probability at least $1-\delta$, $\HIGW$ attains the following regret guarantee:
\begin{align}
    \text{Reg}(T)\!\leq\!\ordO\Bigg(&\!\sum_{t=\tau_1+1}^{T}\!\sqrt{K\xi\!\left(\!\tau_{m(t)-1}-\tau_{m(t)-2},\frac{\delta'}{m(t)^2}\!\right)\!} \notag \\
    &+\sqrt{KB}T\Bigg). \label{eq:main-theorem}
\end{align}
\end{restatable}

Under epoch schedule $\tau_m=2^m$, Theorem \ref{thm:main-theorem} gives
\begin{equation}
    \text{Reg}(T)\le\calO(\sqrt{K\xi(T,\delta'/\log T)}T+\sqrt{KB}T).
\end{equation}
The first term is sublinear in $T$ (typically $\tilde\calO(\sqrt{T})$ with problem-dependent parametric rates on excess risk) since the excess risk rate decreases in the sample size, and the second term is evidently linear in $T$. In the Appendix \ref{app:r-learner-estimation-rate}, we discuss excess risk rates for the $R$-learner oracle under various important function classes. For a more concrete example, consider the following guarantees on the class of linear treatment effect models.

\begin{example}[Linear treatment effect]
    Consider the following class of linear treatment effects
    \begin{equation*}
        \calG=\left\{g:(x,a)\mapsto \langle\theta, x_a\rangle\mid\theta\in\R^d,||\theta||_2\le1\right\}
    \end{equation*}
    where $x=(x_a)_{a\in\calA}$, $x_a\in\R^d$, and $||x_a||_2\le 1$. Under a linear class, the $R$-learner reduces to linear least squares regression of residual outcomes on residual treatments. Therefore, we can use the excess risk rates for square error minimization used in \cite{simchi2021bypassing} to get the same optimal rate guarantees under realizability while obtaining stronger misspecification bounds compared to \cite{krishnamurthy2021adapting} due to our results in Proposition \ref{lem:misspecification-error-bound}. In particular, we have $\text{Reg}(T)\le\calO(\sqrt{KT(d+\log T)} + \sqrt{KB}T)$.
\end{example}
 
Note that Theorem \ref{thm:main-theorem} provides a bias-variance trade-off for contextual bandits. The first term in \eqref{eq:main-theorem} is the regret bound under realizability and it depends on the estimation rate, which captures the variance of the $R$-loss oracle estimate over the class $\calG$. The second term in \eqref{eq:main-theorem} is the regret overhead due to misspecification and it depends on $B$, which captures a notion of bias for the best estimator in the model class $\calG$ under the distribution induced by any probability kernel. For more expressive model classes, the bias term $B$ is small, but the variance term $\xi$ is large, showing that the bounds quantify a bias-variance trade-off for contextual bandits that rely on heterogeneous treatment effect estimation. Moreover, note that Theorem \ref{thm:main-theorem} gives a tighter quantification of the bias-variance trade-off for contextual bandits compared to \cite{krishnamurthy2021adapting} since misspecification under the $R$-learner loss is no greater than misspecification under the squared error loss as shown by Proposition \ref{lem:misspecification-error-bound}.

\paragraph{Limitations}
Our regret bound for the linear model class has worse dependence on the number of arms than the best known upper bound on stochastic contextual bandits \citep{li2019nearly}. This is an inherent limitation of the action-sampling strategy of inverse gap weighting \citet{abe1999associative} which does not make use of the additional structure of the model class. Another limitation is that the exploration rate of our algorithm relies on knowledge of $R$-loss excess risk estimation rates for the specified model class. We discuss instance-dependent estimation rates of the $R$-learner under broad classes of models. However, for any estimator and any model class, fast estimation rates may not be known. Further work on $R$-loss excess risk fast rates or setting the exploration rate in a more data-driven manner would bring practical benefits to our current approach.

\section{EXTENSIONS}\label{sec:extensions}

\paragraph{Model Selection}
Our algorithm depends on specifying a model class with a known estimation rate. We can consider procedures for automatically selecting the appropriate model class in an adaptive manner. In particular, we could easily adapt many methods of model selection \citep{cutkosky2021dynamic, foster2019model, krishnamurthy2021optimal, ghosh2021problem} to our setting. In our second experiment in Section \ref{sec:experiments}, we conduct a form of model selection on linear classes using LASSO regularization on the empirical $R$-loss risk objective. Using this, we adaptively set the excess risk rates based on the estimated sparsity (i.e., complexity) of the model.

\paragraph{Non-stationary Confounder}
Suppose the underlying action-independent component of the true conditional mean reward varied with time: $\starf_t(x,a)=\starg(x,a)+\starh_t(x)$ where $\starh_t$ is a time-varying function of contexts. This scenario can occur in practice where there are non-stationary effects that are independent of treatment recommendations to experimental units such as time-of-day effects on digital platform traffic. Non-stationary confounding factors can make reward model estimation more difficult and thus affect bandit performance as noted in \citet{qin2022adaptivity}. Our algorithm is more robust to such scenarios since it is able to estimate a stationary treatment effect when it exists.

\section{EXPERIMENTS}\label{sec:experiments}

In our experiments, we compare the performance of the $\HIGW$ algorithm with the similar inverse gap weighting algorithm (referred to as $\IGW$ here) that uses an offline least squares regression oracle to model the conditional reward function \citep{simchi2021bypassing,krishnamurthy2021adapting}. The reason is to isolate the benefits of using treatment effect models to make fair comparisons.

\subsection{Setup}

We describe the common setup across experiments. Define $\calS^{d}\coloneqq\{x\in\R^d : ||x||_2=1\}$ with $d=100$. For simplicity, we set $\calX=\calS^d$ and $\calA=\{1,2\}$. At each time step $t\le T$ for $T=10,000$, the learner observes an independently sampled context $x_t\sim\Uniform(\calX)$, takes an action $a_t\in\calA$, and observes a reward $r_t(a_t)=\starf(x_t,a_t)+\varepsilon_t=\starg(x_t,a_t)+\starh(x_t)+\varepsilon_t$ where $\varepsilon_t\sim\sqrt{12\sigma^2}\times\Uniform([-1/2,1/2])$ independently with $\sigma=1/10$.
The specification of $\starf=\starg+\starh$ is provided in each scenario below. We estimate the conditional mean realized reward nuisance parameter $\hatmu$ using cross-fitting. Notably, in every single experiment, we specify the linear function class $\calG=\{g:(x,a)\mapsto\langle\theta,x_a\rangle\mid\theta\in\R^d\}$ for model estimation.
We plot the total regret curve for each experiment. See the figures below.
The shaded areas represent a standard deviation from the mean curve over 25 seed runs for each experiment.

\subsection{Linear Reward \& Linear Treatment Effect}

First, we consider the setting where the conditional mean reward is linear in the contexts. This covers the baseline case where the reward and treatment effect have the same complexity and they are well-specified.
We set $\starf(x,a)=1+\langle\theta_a, x\rangle$ where $\theta_a\sim\Uniform(\calS^{d})$ independently for all $a\in\calA$. In this case, $\starh\equiv 0$.

We find that both algorithms have almost exactly the same performance. Both algorithms exhibit similar sublinear total regret.
This validates the observation that when the treatment effect and reward are equivalent, $\HIGW$ performs just as well as a reward-based method under realizability. See Figure \ref{fig:fig1}.

\begin{figure}[!htb]
  \centering
  \includegraphics[width=.85\linewidth]{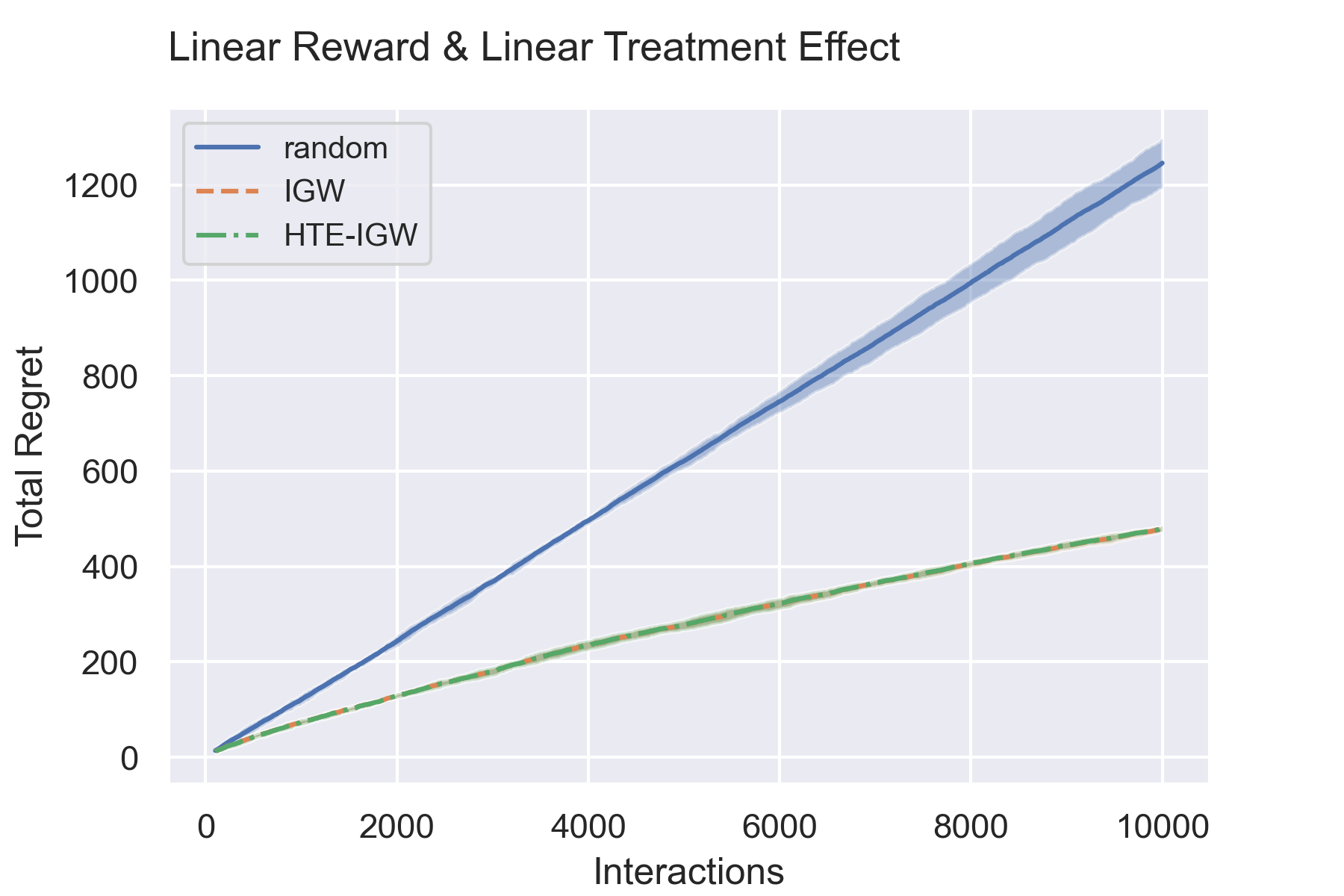}
  \caption{Total regret of synthetic experiment with linear reward and linear treatment effect.}
  \label{fig:fig1}
\end{figure}

\subsection{Linear Reward \& Constant Treatment Effect}

Next, we consider the setting where the treatment effect model lies in a simpler model class than the reward model, yet model realizability still holds for both the treatment effect and reward models. Specifically, the treatment effects are constant and independent of context with a large gap between the treatments, and the confounding term is linear in contexts. Clearly, a linear model class is capable of capturing both the constant treatment effect and linear reward model. We set $\starg(x,a)=u_a + \I(a=1)$ where $u_a\sim\Uniform([0,1])$ independently for all $a\in\calA$ and $\starh(x)=1+\langle\theta, x\rangle$ where $\theta\sim\Uniform(\calS^{d})$.

We find again that $\HIGW$ performs no worse than $\IGW$. In this case, the treatment effect model estimation procedure using a linear class is not capturing the sparsity of the true function. Therefore, we consider model selection to enforce simpler models and readjust exploration parameters according to the model complexity. See top figure in Figure \ref{fig:fig2}.

\paragraph{Model Selection}
We consider the same problem setup, but we use LASSO regularization as well on the risk objective. We also set the exploration rates based on estimated sparsity of the model. Since there is a form of model selection being used here, we refer to the respective algorithms as $\MHIGW$ and $\MIGW$ to make a distinction on their additional adaptivity beyond the base version of these algorithms.

In this case, we find that $\MHIGW$ significantly outperforms $\MIGW$. The $R$-loss oracle quickly finds that none of the context features are predictive of the treatment effects and zeros out all the coefficients besides the bias in the model. Since the treatment effect gap is large, $\MHIGW$ uses a simple model to very quickly find the optimal action and thereby induces a less aggressive exploration parameter. On the other hand, the least squares oracle finds that all the context features are predictive of the rewards since the confounding term has non-zero coefficients for each context. This induces a more aggressive exploration parameter to ensure the algorithm learns an accurate model of a more complex reward function. Essentially, $\MIGW$ over-explores on a less accurate model over longer period of time. We find that although both $\MIGW$ and $\MHIGW$ have sublinear total regret in this example, there is a large gap between them with $\MHIGW$ showing significant improvements in performance. This validates that when the treatment effect model lies in a simpler class than the reward model, treatment effect-based methods (with model selection) can vastly outperform the reward-based method. See bottom figure in Figure \ref{fig:fig2}.

\begin{figure}[!htb]
\centering
\begin{subfigure}[b]{.85\linewidth}
   \includegraphics[width=\linewidth]{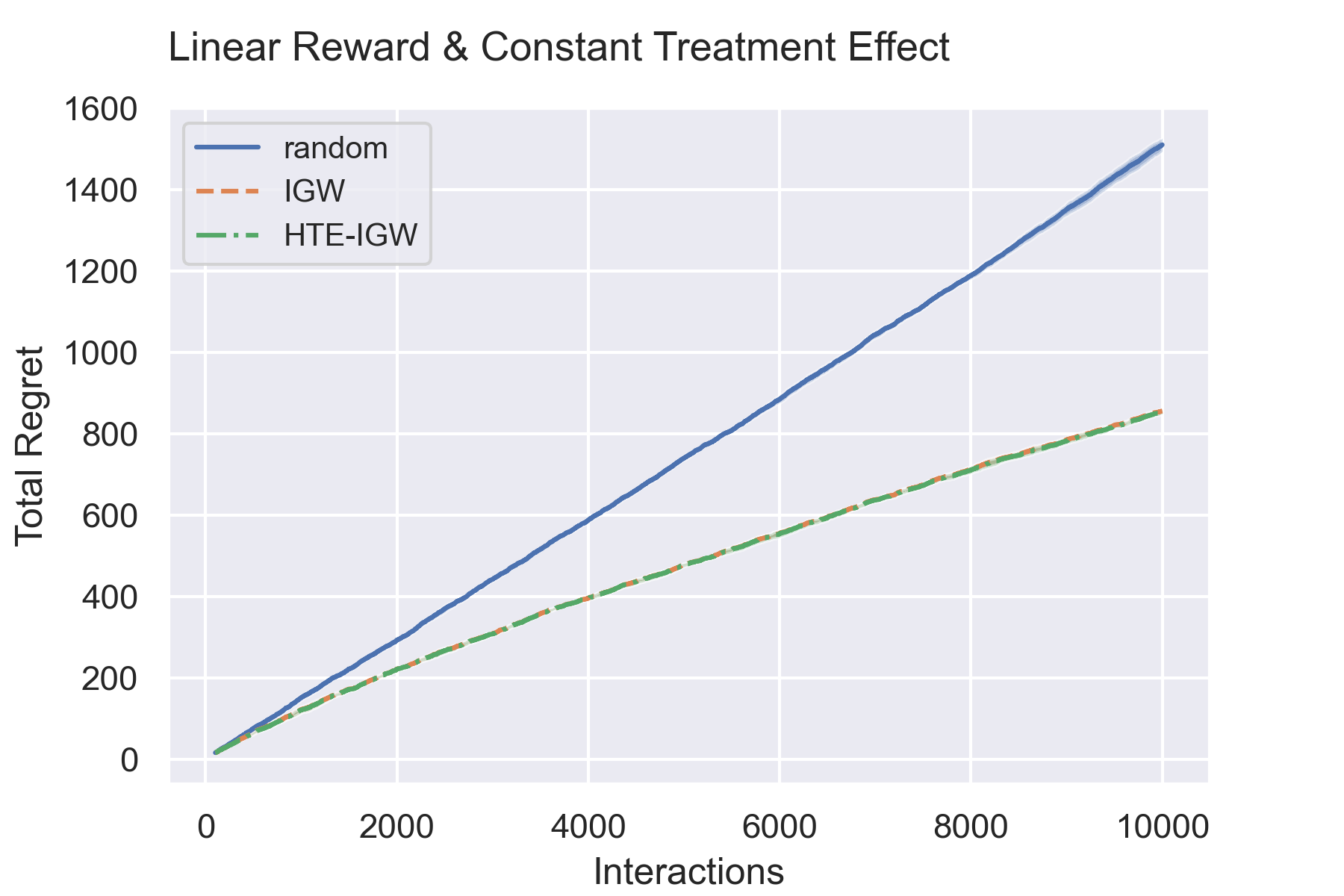}
   \label{fig:fig2sub1} 
\end{subfigure}
\begin{subfigure}[b]{.85\linewidth}
   \includegraphics[width=\linewidth]{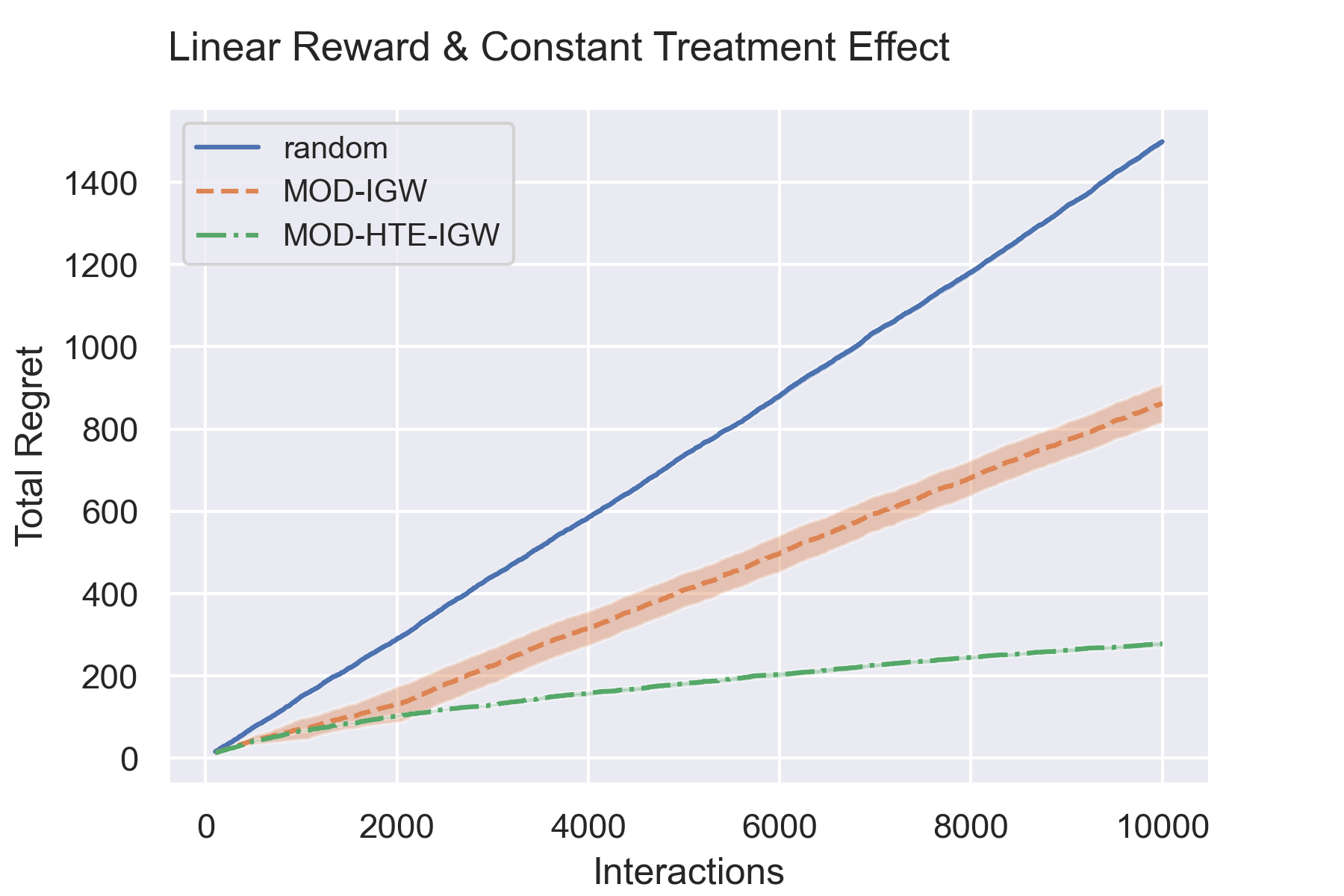}
   \label{fig:fig2sub2}
\end{subfigure}
\caption{Total regret of synthetic experiment with linear reward and constant treatment effect. Top: w/o model selection. Bottom: w/ model selection.}
\label{fig:fig2}
\end{figure}

\subsection{Step-wise Linear Reward \& Linear Treatment Effect}

Lastly, we consider the setting where the treatment effect model lies in a simpler model class than the reward model, and the reward model is misspecified. Specifically, the treatment effects are linear and the confounding term is step-wise linear. The linear model class is capable of capturing the linear treatment effect model but not the step-wise linear reward model. We set $\starg(x,a)=u_a + \I(a=1)$ where $u_a\sim\Uniform([0,1])$ independently for all $a\in\calA$, and $\starh(x)=-\mathbb{I}(x_0 > 1/4)\cdot\max_a(1+\langle\theta_a, x\rangle)$ where $\theta_a\sim\Uniform(\calS^{d})$.

We find that due to misspecification of the reward model, the total regret of $\IGW$ is not much better then random sampling, but $\HIGW$ manages to make noticeable improvements. See Figure \ref{fig:fig3}.

\begin{figure}[h]
  \centering
  \includegraphics[width=.85\linewidth]{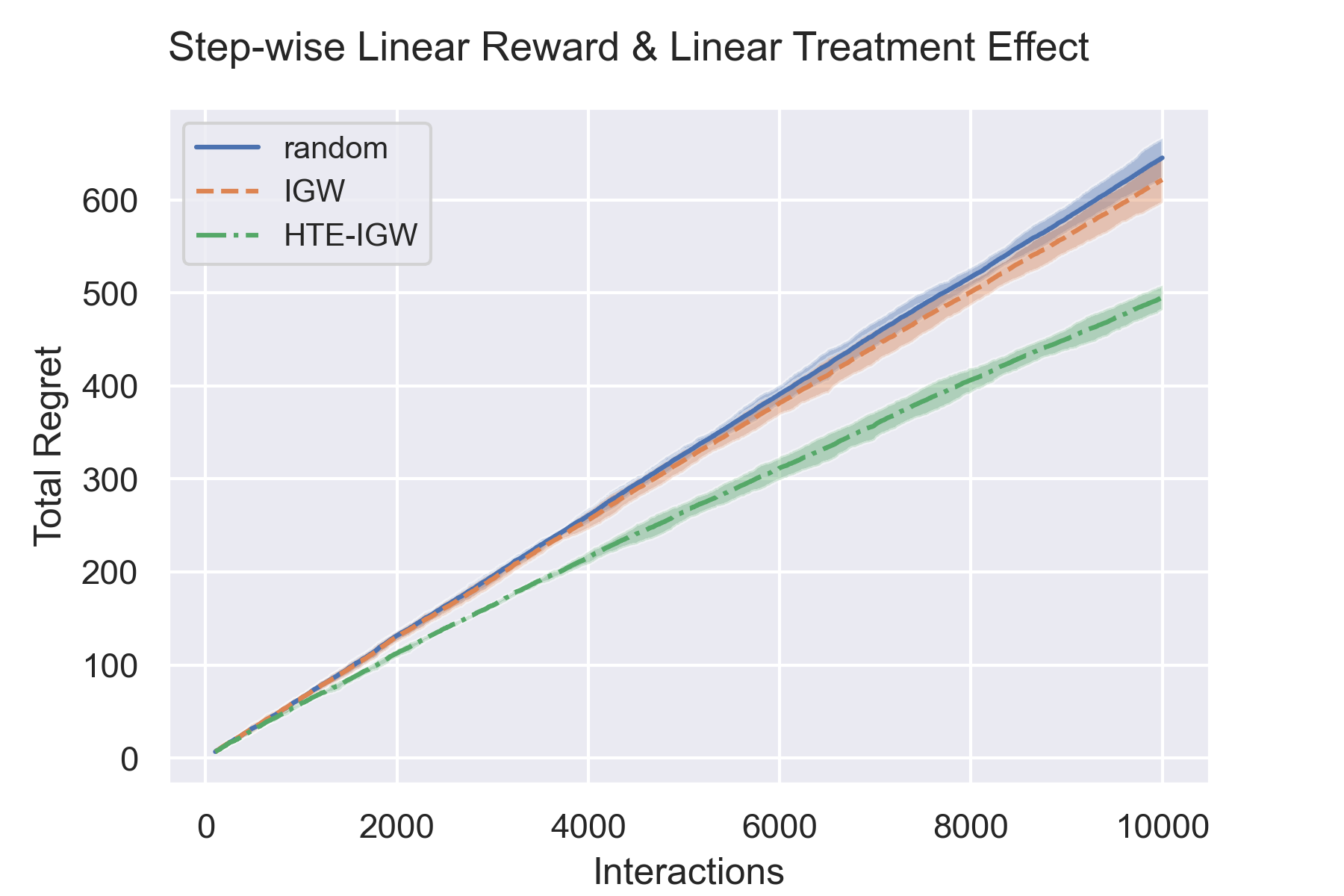}
\caption{Total regret of synthetic experiment with step-wise linear reward and linear treatment effect.}
\label{fig:fig3}
\end{figure}

\subsection{Perturbed Step-wise Linear Reward \& Perturbed Linear Treatment Effect}
We present an additional experiment for the case where the treatment effect model is also misspecified. We use a linear model class while the true treatment effects are linear with a sinusoidal perturbation and the confounding term is step-wise linear with a sinusoidal perturbation. Therefore, the true treatment effect and reward models are non-linear and misspecified under the linear model class. We set $g^*(x,a)=\mathbb{I}(a=1)+\langle\theta_a, x\rangle + \sin(\langle\theta_a, x\rangle)$ and $h^*(x)=-\mathbb{I}(x_0 > 1/4)\cdot\max_a(1+\langle\theta_a, x\rangle)$ where $\theta_a\sim\Uniform(\calS^d)$ independently for all $a\in\mathcal{A}$. We use a linear model class (with and without model selection). We observe that the performance of both $\HIGW$ and $\IGW$ suffer from model misspecification. However, with model selection, $\HIGW$ regains the performance, showing that $\HIGW$ is more capable of capturing the simplest model class necessary for decision making even under model misspecification. See Figure \ref{fig:fig4}.

\begin{figure}[!htb]
\centering
\begin{subfigure}[b]{.85\linewidth}
   \includegraphics[width=\linewidth]{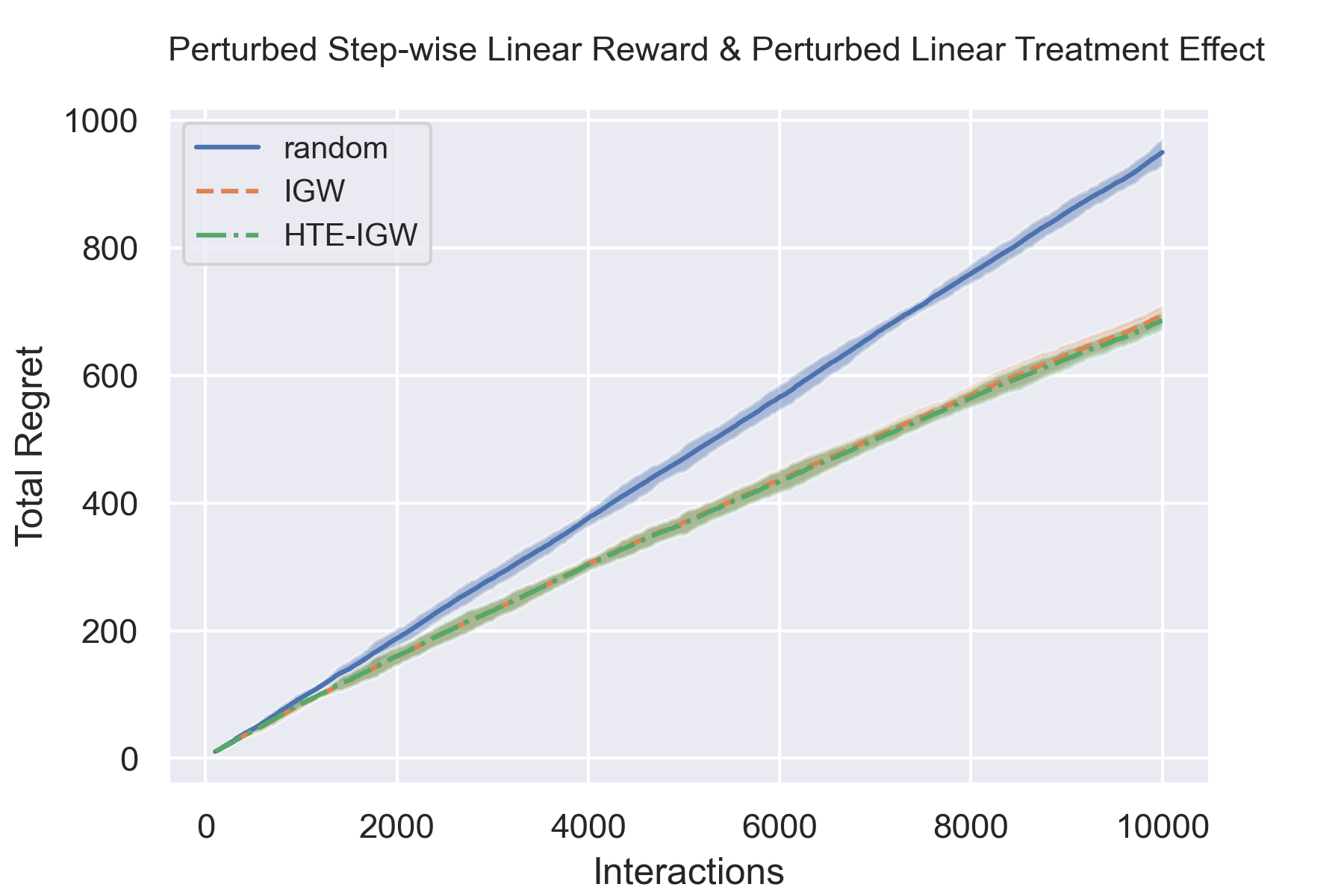}
   \label{fig:fig4sub1} 
\end{subfigure}
\begin{subfigure}[b]{.85\linewidth}
   \includegraphics[width=\linewidth]{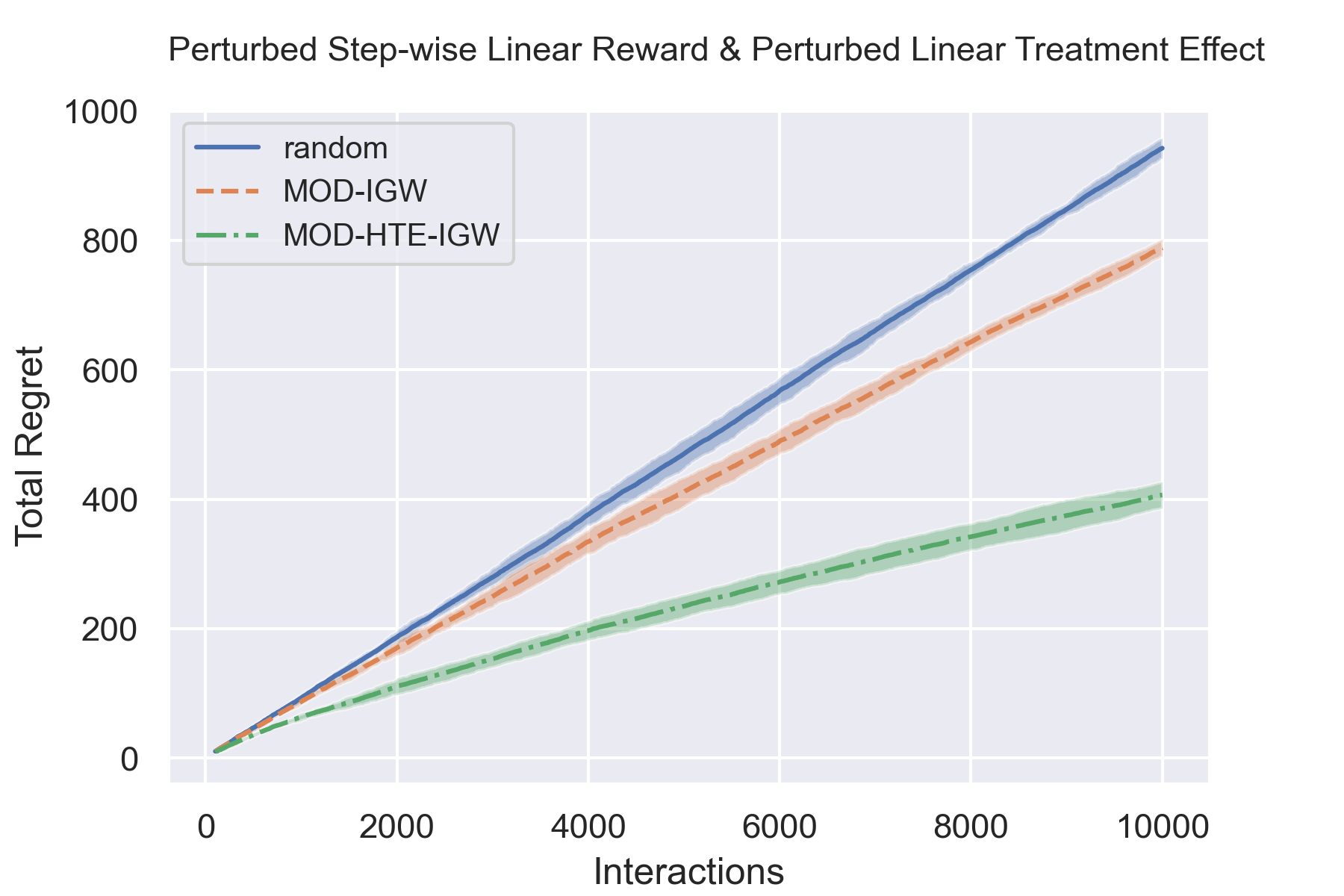}
   \label{fig:fig4sub2}
\end{subfigure}
\caption{Total regret of synthetic experiment with non-linear (sinusoidally perturbed step-wise linear) reward and linear treatment effect. Top: w/o model selection. Bottom: w/ model selection.}
\label{fig:fig4}
\end{figure}

\section{CONCLUSION}


In this paper, we presented the first universal reduction of cumulative regret minimization in stochastic contextual bandits to heterogeneous treatment effect estimation method via $R$-loss minimization. Based on this reduction, we develop a computationally efficient, flexible, general-purpose algorithm that leverages techniques presented in \citep{krishnamurthy2021adapting} to not require treatment effect realizability and to be more robust to misspecification. Furthermore, we showed that treatment effect estimation oracles are a more robust primitive for contextual bandits than squared error regression oracles as commonly used in reward estimation \citep{foster2020beyond, simchi2021bypassing}. Beyond cumulative regret minimization, heterogeneous treatment effect estimation should also lead to better contextual bandit algorithms for simple regret minimization---which still appears to be a challenging problem based on existing algorithms \citep{athey2022contextual}. We hope our findings will inspire further research on adapting econometric methods of treatment effect estimation for developing more robust and practical contextual bandits. 

\vfill

\subsubsection*{Acknowledgements}
The authors would like to thank Undral Byambadalai, Stefan Wager, Ruohan Zhan, and the reviewers for their helpful feedback. The authors are also grateful for the generous support provided by Golub Capital Social Impact Lab. S.A. and S.K.K. also acknowledge generous support from the Office of Naval Research grant N00014-19-1-2468.




\bibliographystyle{plainnat}
\bibliography{main}

\appendix
\onecolumn

\section{R-LOSS EXCESS RISK, MISSPECIFICATION, AND ESTIMATION RATES}
\label{app:r-loss-details}

\subsection{Equivalent Form of Excess Risk}
\label{app:excess-risk-id}

Proposition \ref{lem:r-excess-risk} states an equivalent form for the $R$-loss excess risk. Notice how this form is independent of the conditional mean realized reward estimate $\hatmu$ that is fixed with respect to the data-generating distribution (as in a constant or cross-fitted estimate). Proposition \ref{lem:r-excess-risk} will later allow us to prove Proposition \ref{lem:misspecification-error-bound} and to bound in Lemma \ref{lem:reg-est-accuracy} the error of the treatment effect estimate using bounds on the excess risk of the $R$-loss oracle estimate.
\lemExcessRisk*
\begin{proof}
    Consider any $g:\calX\times\calA\to[0,1]$.
    For conciseness in this proof, we shorten the notation $g(x)\coloneqq g(x,\cdot)=(g(x,a))_{a\in\calA}$. First of all, we have
     \begin{align*}
        \calR_p(g)&=\E\left[(r(a)-\hatmu(x)-\langle e_a-p(x), g(x)\rangle)^2\right] \\
        &=\E[(r(a)-\starf(x,a)+\starf(x,a)-\hatmu(x)-\langle e_a-p(x),g(x)\rangle)^2] \\
        &\ \begin{aligned}
            =&\E[(r(a)-\starf(x,a))^2]+\E[(\starf(x,a)-\hatmu(x)-\langle e_a-p(x),g(x)\rangle)^2] \\
            &-2\E[(r(a)-\starf(x,a))\cdot(\starf(x,a)-\hatmu(x)-\langle e_a-p(x),g(x)\rangle)].
        \end{aligned}
    \end{align*}
    Since $\hatmu$ is independent of $(x,a,r(a))$, the last term of this last equation is
    \begin{align*}
        &\E[(r(a)-\starf(x,a))\cdot(\starf(x,a)-\hatmu(x)-\langle e_a-p(x),g(x)\rangle)] \\
        &=\E_{x,a}[\E_r[(r(a)-\starf(x,a))\cdot(\starf(x,a)-\hatmu(x)-\langle e_a-p(x),g(x)\rangle)\mid x, a]] \\
        &=\E_{x,a}[(\E_r[r(a)|x,a]-\starf(x,a))\cdot(\starf(x,a)-\hatmu(x)-\langle e_a-p(x),g(x)\rangle)] \\
        &=\E_{x,a}[(\starf(x,a)-\starf(x,a))\cdot(\starf(x,a)-\hatmu(x)-\langle e_a-p(x),g(x)\rangle)] \\
        &=0
    \end{align*}
    and the second term is
    \begin{align*}
        &\E[(\starf(x,a)-\hatmu(x)-\langle e_a-p(x),g(x)\rangle)^2] \\
        &=\E\left[\left(\left(\starf(x,a)-\E_{a'\sim p(\cdot|x)}[\starf(x,a')]\right)-\langle e_a-p(x),g(x)\rangle+\left(\E_{a'\sim p(\cdot|x)}[\starf(x,a')]-\hatmu(x)\right)\right)^2\right] \\
        &=\E\left[\left(\langle e_a-p(x),\starf(x)\rangle-\langle e_a-p(x),g(x)\rangle+\Big(\mu(x)-\hatmu(x)\Big)\right)^2\right] \\
        &=\E\left[\left(\langle e_a-p(x),\starf(x)-g(x)\rangle+\Big(\mu(x)-\hatmu(x)\Big)\right)^2\right] \\
        &=\E\left[\langle e_a-p(x),\starf(x)-g(x)\rangle^2\right]+\E[(\mu(x)-\hatmu(x))^2]+2\E[(\mu(x)-\hatmu(x))\cdot\langle e_a-p(x),\starf(x)-g(x)\rangle] \\
        &=\E\left[\langle e_a-p(x),\starf(x)-g(x)\rangle^2\right]+\E[(\mu(x)-\hatmu(x))^2]+2\E_x[\E_{a\sim p(\cdot|x)}[(\mu(x)-\hatmu(x))\cdot\langle e_a-p(x),\starf(x)-g(x)\rangle\mid x]] \\
        &=\E\left[\langle e_a-p(x),\starf(x)-g(x)\rangle^2\right]+\E[(\mu(x)-\hatmu(x))^2]+2\E_x[(\mu(x)-\hatmu(x))\cdot\E_{a\sim p(\cdot|x)}[\langle e_a-p(x),\starf(x)-g(x)\rangle\mid x]] \\
        &=\E\left[\langle e_a-p(x),\starf(x)-g(x)\rangle^2\right]+\E[(\mu(x)-\hatmu(x))^2]+2\E_x[(\mu(x)-\hatmu(x))\cdot\langle p(x)-p(x),\starf(x)-g(x)\rangle] \\
        &=\E\left[\langle e_a-p(x),\starf(x)-g(x)\rangle^2\right]+\E[(\mu(x)-\hatmu(x))^2].
    \end{align*}
    
    Therefore,
    \begin{align*}
        \calR_p(g)&=\E[(r(a)-\starf(x,a))^2]+\E\left[\langle e_a-p(x),\starf(x)-g(x)\rangle^2\right]+\E[(\mu(x)-\hatmu(x))^2].
    \end{align*}
    This implies that the set of true treatment effect models $\argmin_{g'}\calR_p(g')$ are the functions $g$ that satisfy $\langle e_a-p(x),g(x)\rangle=\langle e_a-p(x),\starf(x)\rangle$ for any $(x,a)\in\calX\times\calA$. This is the set of true treatment effect functions
    \begin{align*}
        \calG^*=\{g:\calX\times\calA\to[0,1]\mid\exists h\!:\!\calX\to[0,1]\text{ s.t. }g(x,a)+h(x)=\starf(x,a), \forall (x,a)\in\calX\times\calA\Big\}.
    \end{align*}
    Therefore, for any $\starg\in\argmin_{g'}\calR_p(g')$,
    \begin{align*}
        \calE_p(g)&=\calR_p(g)-\calR_p(\starg) \\
        &=\E[\langle e_a-p(x),\starf(x)-g(x)\rangle^2]-\E[\langle e_a-p(x),\starf(x)-\starg(x)\rangle^2] \\
        &=\E[\langle e_a-p(x),\starg(x)-g(x)\rangle^2].
    \end{align*}
\end{proof}

Note that from the proof we can see that this form of the excess risk is equivalent to the excess risk that uses the true nuisance parameters inside the loss function. Since the rest of our analysis relies on bounds on the excess risk, this shows that we lose nothing by relying on this form of the $R$-loss risk in our analysis, assuming the estimate $\hatmu$ of $\mu$ is fixed.


\subsection{Bound on Misspecification Error}
\label{app:misspecification-bound}

Proposition \ref{lem:misspecification-error-bound} upper bounds the excess risk of the $R$-learner loss by the excess risk of the squared error loss.

\lemMisspecificationBound*
\begin{proof}
    Let $g\in\calG$ and $\starg\in\argmin_{g'}\calR_p(g')$. Also, let $\Delta_{x,a}\coloneqq g(x,a)-\starg(x,a)$. Then, from Proposition \ref{lem:r-excess-risk} we have
    \begin{align*}
        \calE_p(g)&=\E\left[\langle e_a-p(x), g(x,\cdot)-\starg(x,\cdot)\rangle^2\right] \\
        &=\E\left[\left(\big(g(x,a)-\starg(x,a)\big)-\E_{a'\sim p(\cdot|x)}[g(x,a')-\starg(x,a')]\right)^2\right] \\
        &=\E\left[\left(\Delta_{x,a}-\E_{a'\sim p(\cdot|x)}[\Delta_{x,a'}]\right)^2\right] \\
        &=\E\left[\Delta_{x,a}^2+\E_{a'\sim p(\cdot|x)}[\Delta_{x,a'}]^2-2\E_{a'\sim p(\cdot|x)}[\Delta_{x,a'}]\cdot\Delta_{x,a}\right] \\
        &=\E\left[\Delta_{x,a}^2\right]+\E_{x\sim D_\calX}\left[\E_{a\sim p(\cdot|x)}[\Delta_{x,a}]^2\right]-2\E_{x\sim D_\calX}\left[\E_{a'\sim p(\cdot|x)}[\Delta_{x,a'}]\cdot\E_{a\sim p(\cdot|x)}\big[\Delta_{x,a} \ \big|\ x\big]\right] \\
        &=\E\left[\Delta_{x,a}^2\right]+\E_{x\sim D_\calX}\left[\E_{a\sim p(\cdot|x)}[\Delta_{x,a}]^2\right]-2\E_{x\sim D_\calX}\left[\E_{a\sim p(\cdot|x)}[\Delta_{x,a}]^2\right] \\
        &=\E\left[\Delta_{x,a}^2\right]-\E_{x\sim D_\calX}\left[\E_{a\sim p(\cdot|x)}[\Delta_{x,a}]^2\right] \\
        &=\E\left[\big(g(x,a)-\starg(x,a)\big)^2\right] - \E_{x\sim D_\calX}\left[\E_{a\sim p(\cdot|x)}\big[g(x,a)-\starg(x,a)\big]^2\right] \\
        &\le\E\left[\big(g(x,a)-\starg(x,a)\big)^2\right].
    \end{align*}
    Since $\starf\in\argmin_{g'}\calR_p(g')$, this implies that
    \begin{align*}
        B\le\max_{p\in\calP}\min_{f\in\calG}\E\left[\big(f(x,a)-\starf(x,a)\big)^2\right].
    \end{align*}
\end{proof}

In other words, for a chosen model class $\calG$, the average misspecification error of the $R$-loss risk is upper bounded by the average misspecification error of the square error risk. A square error oracle is usually used to estimate the reward model in \citet{simchi2021bypassing}. This implies that heterogeneous treatment effect estimation via $R$-loss oracles is more robust to model misspecification than reward estimation via least squares oracles. This can be interpreted to mean that heterogeneous treatment effect estimation is no harder than reward estimation.

\paragraph{Misspecification Error Bound Intuition}

We provide further intuition to the previous result.
Essentially, the set of true treatment effects
\begin{equation*}
    \calG^*\coloneqq\argmin_{g'}\calR_p(g')=\{g:\calX\times\calA\to[0,1]\mid\exists h:\calX\to[0,1]\text{ s.t. } g(x,a)+h(x)=f^*(x,a),\forall(x,a)\in\calX\times\calA\}
\end{equation*}
specifies a buffer zone around the true reward model $\starf$ of sufficiently valid estimands for decision-making. This set contains infinitely many possible target estimands, including the true reward model $f^*$. Therefore, the ``distance" $B_R$ from any fixed model class $\calG$
to the ``nearest" estimand $g^*$
(corresponding to the misspecification error under the $R$-loss, i.e., the left hand side of inequality \eqref{eq:misspecification-error-bound})
in this buffer zone must necessarily be no further than the ``distance" $B_{\text{sq}}$ to the true reward model $f^*$ (corresponding to the misspecification error under the squared error loss, i.e., the right-hand side of inequality \eqref{eq:misspecification-error-bound}).
See Figure \ref{fig:R-loss-intuition} for a graphical representation of this observation.
Note that $\calG$ may possibly not contain $\starf$ even while at the same time possibly still overlap with $\calG^*$, in which case the treatment effect model would be well-specified while the reward model would be misspecified. If the reward model is well-specified, then the treatment effect model is well-specified. However, the converse may not necessarily be true.

\begin{figure}[h]
\centering
\includegraphics[width=.9\linewidth]{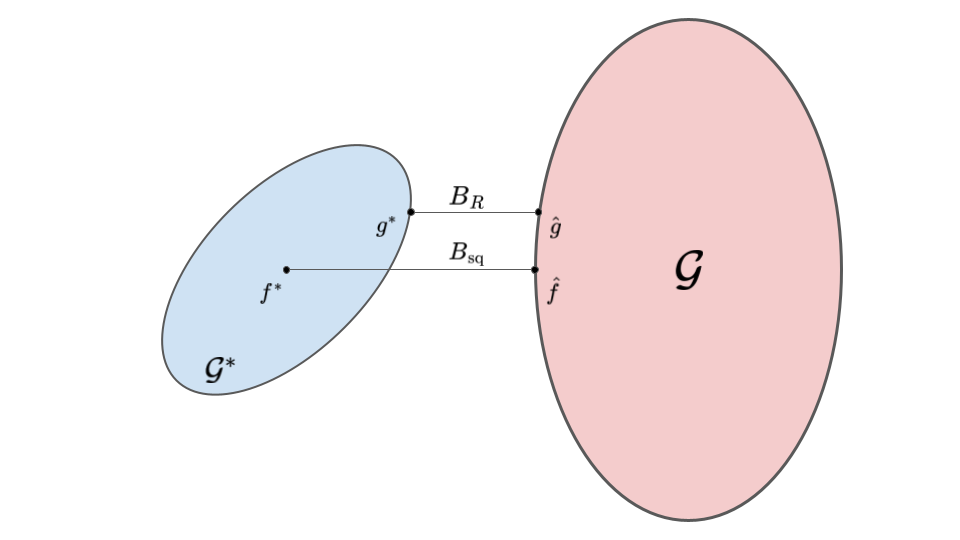}
\caption{Intuition behind differences between $R$-loss vs.~squared error loss misspecification error. $B_R$ denotes the misspecification error of the $R$-loss, and $B_{\text{sq}}$ denotes the misspecification error of the squared error loss.}
\label{fig:R-loss-intuition}
\end{figure}

\subsection{R-learner Oracle Estimation Rates}
\label{app:r-learner-estimation-rate}

\citet{bartlett2005local} and \citet{xu2020towards} discuss problem-dependent rates on the excess risk of a model trained via empirical risk minimization. We leverage their loss-dependent rates \citep[see Equation 4.3 in][]{xu2020towards} and upper bound them with probability $1-\delta$ as
\begin{align}
    \calE_p(\hatg)\le\calO\left(r^* \vee \sqrt{Br^*}\right) \le \calO\left(r^*+B\right)
\end{align}
where $B$ is the average misspecification error under the given loss and $r^*$ is the fixed point of a ``surrogate'' function $\psi(r;\delta)$ that upper bounds the uniform error within a localized region of the model class. Note that we allow for a loose bound on the dependence on $B$ because our regret guarantees are stated with an additive regret overhead due to misspecification that depends linearly on $B$, so this error can be absorbed into the misspecification overhead. This also makes it clear how misspecification issues can be separated from estimation in our algorihtm since the bound on regret due to misspecification is covered by the safety check guarantees and so misspecification does not need to shape the exploration strategy to achieve our regret bounds.

Clearly, the fixed point depends on the model and the probability level $\delta$. However, note that the following bounds on the fixed point of the surrogate function hide a logarithmic dependence on $\delta$, typically of the form $\log(1/\delta)$ for the sake of simplicity as well. For VC-type classes defined by the class of models whose metric entropy is $\calO(d\log\frac{1}{\varepsilon})$ for any covering net size $\epsilon>0$ where $d$ is the VC dimension, it can be proven \citep{koltchinskii2000rademacher, xu2020towards} that $r^*\le\calO(\frac{d\log n}{n})$. Therefore, VC-type classes have $R$-loss excess risk bound $\calO(\frac{d\log n}{n}+ B)$. Note that linear model classes with $d$-dimensional coefficients are VC-type classes with VC dimension equal to $d$. Similarly, for any general non-parametric class of models whose metric entropy is $\calO(\varepsilon^{-2\rho})$ for any covering net size $\epsilon>0$ where $\rho\in(0,1)$ is a constant, it can be proven that $r^*\le\calO(n^{-\frac{1}{1+\rho}})$. Therefore, non-parametric classes of polynomial growth have $R$-loss excess risk bound $\calO(n^{-\frac{1}{1+\rho}}+B)$.

\section{ALGORITHM ANALYSIS}
\label{app:analysis}

\subsection{Definitions}

Recall we have a class of treatment effect models $\calG\subset\{g:\calX\times\calA\to[0,1]\}$ and the set of true treatment effect models is $\starcalG\coloneqq\argmin_{g'}\calR_p(g')$ where the infimum is taken with respect to all functions. The following analysis is stated with respect to a fixed $\starg\in\starcalG$. This is without a loss of generalization because the subsequent results are the same independent of the choice of $\starg\in\starcalG$. With an abuse of notation, for any function $g:\calX\times\calA\to[0,1]$, we let $g(x)$ denote the vector $g(x,\cdot)=(g(x,a))_{a\in\calA}$.

A policy $\pi:\calX\to\calA$ is a deterministic mapping from contexts to actions. Let $\Psi = \A^{\Xscript}$ be the universal policy space containing all possible policies. The policy $\pi_g$ that is induced by a model $g\in\calG$ is given by $\pi_g(x)\coloneqq\argmax_a g(x,a)$ for every $x\in\calX$. Let $\pi^*$ denote the policy induced by $\starg$, i.e., $\pi^*\coloneqq\argmax_a g^*(x,a)$. The expected instantaneous treatment effect of a policy $\pi\in\Psi$ with respect to a treatment effect model $g$ and probability kernel $p:\calX\to\Delta_\calA$ is defined as
\begin{equation}\label{eq:instantaneous-effect}
    \Phi_p^g(\pi) \coloneqq \E_{x \sim D_\Xscript}[\langle e_{\pi(x)}-p(x), g(x)\rangle]=\E_{x\sim D_\calX}[g(x,\pi(x))-\sum_{a\in\calA}p(a|x)g(x,a)].
\end{equation}
Moreover, we write $\Phi_p(\pi)$ to mean $\Phi^{g^*}_p(\pi)$, the true expected instantaneous treatment effect for policy $\pi$ under probability kernel $p$. We define the expected instantaneous regret of a policy $\pi$ with respect to the treatment effect model $g$ as
\begin{equation}\label{eq:instantaneous-regret}
    \Reg_g(\pi)\coloneqq \Phi_p^g(\pi_g)-\Phi_p^g(\pi)=\E_{x\sim D_\calX}[g(x,\pi_g(x))-g(x,\pi(x))].
\end{equation}
Note that by this definition the expected instantaneous regret for any policy is independent of the choice of probability kernel. We write $\Reg(\pi)$ to mean $\Reg_{g^*}(\pi)$, the true expected instantaneous regret for policy $\pi$.

We also consider the true expected instantaneous reward of a policy $\pi$ defined as
\begin{equation}\label{eq:instantaneous-reward}
    R(\pi)\coloneqq\E_{x\sim D_\calX}[f^*(x,\pi(x))].
\end{equation}
Note that, according to the discussion in Appendix \ref{app:excess-risk-id} on the set of true treatment effect models, the following equivalences for the true expected instantaneous regret hold:
\begin{align}\label{eq:true-instantaneous-regret}
    \Reg(\pi)&=\Phi_p(\pi^*)-\Phi_p(\pi)=\E_{x\sim D_\calX}[g^*(x,\pi^*(x))-g^*(x,\pi(x))] \\
    &=R(\pi^*)-R(\pi)=\E_{x\sim D_\calX}[f^*(x,\pi^*(x))-f^*(x,\pi(x))].
\end{align}

\begin{remark}\label{rmk:equiv-def}
Although it is constructed differently through treatment effect models, this definition of the true expected instantaneous regret $\Reg(\pi)$ in fact is equal to the true expected instantaneous regret in \cite{simchi2021bypassing, krishnamurthy2021adapting}. This allows us to carry out more or less a similar analysis in our proofs.
\end{remark}

Next, for any probability kernel $p$ and any policy $\pi$, we let $V(p,\pi)$ denote the expected inverse probability of policy $\pi$ under kernel $p$, \footnote{In \citet{simchi2021bypassing}, this term is called the decisional divergence between the randomized policy $Q_p$ and deterministic policy $\pi$.} defined as
\begin{equation}\label{eq:decisional-divergence}
    V(p,\pi)\coloneqq\E_{x\sim D_{\Xscript}}\left[\frac{1}{p(\pi(x)|x)}\right].
\end{equation}
This is a crucial quantity in our analysis that measures the overlap between a policy and a data-sampling kernel.
Next, we let $\Gamma_t$ denote the set of observations up to and including time $t$, that is,
\begin{equation}
    \Gamma_t\coloneqq\left\{(x_s,a_s,r_s(a_s))\right\}_{s=1}^t.
\end{equation}
Lastly, we let $m(t)$ denote the epoch associated to round $t$.

\subsection{Action Selection Kernel as a Randomized Policy}
Given any probability kernel $p$, from Lemma 3 in \citet{simchi2021bypassing}, there exists a unique product probability measure on $\Psi$, given by:
\begin{equation}
    \label{eq:q-product}
    Q_p(\pi) \coloneqq \prod_{x\in\Xscript} p(\pi(x)|x).
\end{equation}
This measure satisfies the following property
\begin{align}
    \label{eq:connecting-p-and-Q}
  p(a|x) = \sum_{\pi\in\Psi} \I\{\pi(x)=a\}Q_p(\pi).
\end{align}
Since any probability kernel $p$ induces the distribution $Q_p$ over the set of deterministic policies $\Psi$, we can think of $Q_p$ as a randomized policy induced by $p$. Equations \eqref{eq:connecting-p-and-Q} and \eqref{eq:q-product} establish a correspondence between the probability kernel $p$ and the induced randomized policy $Q_p$.
For short hand, we let $Q_m\equiv Q_{p_m}$ for any epoch $m$.

\subsection{Safe Epoch}
We say an epoch $m$ is safe when the status of the algorithm at the end of the epoch is safe, that is, when the variable $\safe$ is still set to $\true$ at the end of this epoch. Let $\msafealg$ be the last safe epoch determined by the $\CheckAndChooseSafe$ subroutine in Algorithm 1. Note that, for all $m\leq \msafealg$, the epoch $m$ is safe. Now let $\msafe$ be such that:
\begin{equation}
    \label{eq:msafe-2}
    \begin{aligned}
      \msafe\coloneqq\max\left\{ m \mid C_0B \leq \xi\Big(\tau_m-\tau_{m-1}, \frac{\delta'}{m^2} \Big) \right\}.
    \end{aligned}
\end{equation}
The epoch $\msafe$ is critical to our theoretical analysis. It is the last epoch before the estimation rate of the $R$-loss oracle goes below the misspecification level.

\subsection{High Probability Event}
\label{sec:high-prob-events-1}

Theorem \ref{thm:main-theorem}
provides certain high probability bounds on cumulative regret. As a preliminary step in these proofs, it will be helpful to show that the event $\eventReg$ defined below hold with high probability. Event $\eventReg$ ensures that our treatment effect estimates are ``good" models for the first few epochs. This describes the event where, for any epoch $m\le\min(m^*,\hatm)$, the $R$-loss risk of the estimated model $\hatg_{m+1}$ can be bounded purely in terms of the known estimation rate of the algorithm:
\begin{equation}
    \label{eq:w-event-Reg}
    \begin{aligned}
      \eventReg\coloneqq\left\{\forall m\le\min(\msafe,\msafealg), \calE_{p_m}(\hatg_{m+1}) \leq 2\xi\left(\tau_m-\tau_{m-1}, \frac{\delta'}{m^2} \right) \right\}.
    \end{aligned}
\end{equation}

\begin{lemma}\label{lem:high-prob-eventReg}
    Suppose the $R$-loss oracle satisfies Assumption \ref{ass:main-assumption}.
    Then the event $\eventReg$ holds with probability at least $1-\delta$.
\end{lemma}
\begin{proof}
    Consider any epoch $m$ such that epoch $m-1$ was safe. Since epoch $m-1$ was safe, for the first time-step in epoch $m$ the algorithm samples an action from the probability kernel $p_m$. Hence from Assumption \ref{ass:main-assumption}, with probability $1-\delta'/m^2$, we have:
    \begin{equation}
        \label{eq:mse_event}
        \calE_{p_m}(\hatg_{m+1}) \leq \xi\left(\tau_m-\tau_{m-1}, \frac{\delta'}{m^2} \right) + C_0B.
    \end{equation}
    Therefore, the probability that \eqref{eq:mse_event} does not hold for some epoch $m\le\min(\msafe,\msafealg)$ can be bounded by
    \begin{equation*}
        \sum_{m=1}^{\infty} \frac{\delta'}{m^2} \leq \delta.
    \end{equation*}
    Recall $\delta'=\delta/2$.
    Hence, from the definition of $\msafe$, we get that $\eventReg$ holds with probability at least $1-\delta$.
\end{proof}

\subsection{Safety Check Assumptions}\label{app:safety-check-assumptions}

We will assume the $\CheckAndChooseSafe$ subroutine determines the status of the algorithm such that the following assumptions hold. See Algorithms 2 and 3 of \citet{krishnamurthy2021adapting} for an example of a safety check algorithm that guarantees these properties and see Lemmas 9 and 10 for their proofs.
\begin{assumption}\label{ass:msafe-bound}
    Suppose the event $\eventReg$ holds with probability at least $1-\delta$. Then, for $t\le\tau_{\msafe+1}$, the status of the algorithm at the end of round $t$ is safe, and when $t>\tau_{\msafe+1}$, we have that $\msafe+1\le\hatm$ with probability at least $1-\delta$.
\end{assumption}

\begin{assumption}\label{ass:post-msafe-reg-bound}
    Suppose the event $\eventReg$ holds with probability at least $1-\delta$. Then, for any epoch $m\ge\msafe$, the expected implicit regret is upper bounded as
    \begin{equation}
        \sum_{\pi\in\Psi}Q_{p_m}(\pi)\Reg(\pi)\le O(\sqrt{KB})
    \end{equation}
    with probability at least $1-\delta$.
\end{assumption}

Assumption \ref{ass:msafe-bound} guarantees that the safe kernel will not be triggered before true safe epoch $\msafe$. Assumption \ref{ass:post-msafe-reg-bound} guarantees that once the algorithm status is not safe, the expected instantaneous reward is at least controlled by the misspecification error.

\citet{krishnamurthy2021adapting} achieves the bound on the expected implicit regret for every epoch $m\ge\msafe$ as stated in Assumption \ref{ass:post-msafe-reg-bound} by constructing a Hoeffding-style high-probability lower bound on the cumulative reward and triggering a safety check mechanism when the cumulative reward dips below this lower bound. Then, the safety mechanism defaults to the most recent kernel with the largest lower bound. This high-probability misspecification test ensures that the expected regret does not only stay bounded after the test is triggered, but also that the expected regret stays bounded before the test is triggered, even after the safe epoch has occurred.

\subsection{Per-epoch Properties}

Lemma \ref{lem:conditional-regret} forms an equivalence of the expected instantaneous regret of a policy as the expected implicit regret of the randomized policy, which will be useful in bounding the total regret.
\begin{lemma}\label{lem:conditional-regret}
    For any round $t$,
    \begin{equation*}
        \E_{x_t,r_t,a_t}[r_t(\starpi(x_t))-r_t(a_t)|\Gamma_{t-1}] =  \sum_{\pi\in\Psi}Q_{m(t)}(\pi)\Reg(\pi).
    \end{equation*}
\end{lemma}
\begin{proof}
    Consider any time step $t
    \geq 1$, then from \eqref{eq:connecting-p-and-Q} relating kernel $p$ and $Q_p$ and \eqref{eq:true-instantaneous-regret} we have the following equalities:
    \begin{align*}
    \begin{split}
      &\E_{x_t,r_t,a_t}[r_t(\starpi(x_t))-r_t(a_t)|\Gamma_{t-1}]\\
      & = \E_{x\sim D_{\Xscript}, a \sim p_t(\cdot|x)}[\starf(x,\starpi(x))-f^*(x,a)]\\
      & = \E_{x\sim D_{\Xscript}}\Bigg[\sum_{a\in\A}p_t(a|x)(\starf(x,\starpi(x))-f^*(x,a))\Bigg]\\
      & = \E_{x\sim D_{\Xscript}}\Bigg[\sum_{a\in\A}\sum_{\pi\in\Psi}\I(\pi(x)=a)Q_{m(t)}(\pi)(\starf(x,\starpi(x))-f^*(x,a))\Bigg]\\
      & = \sum_{\pi\in\Psi}Q_{m(t)}(\pi)\E_{x\sim D_{\Xscript}}[\starf(x,\starpi(x))-f^*(x,\pi(x))]\\
      & =  \sum_{\pi\in\Psi}Q_{m(t)}(\pi)\Reg(\pi).
    \end{split}
    \end{align*}
\end{proof}

Lemma~\ref{lem:QmRegEst} states a key bound on the estimated implicit regret of the randomized policy $Q_m$.

\begin{restatable}[]{lemma}{lemQmRegEst}
\label{lem:QmRegEst}
    For any epoch $m$,
    $$ \sum_{\pi\in\Psi} Q_m(\pi)\Reg_{\hatg_m}(\pi) \leq \frac{K}{\gamma_m}. $$
\end{restatable}
\begin{proof}
    This follows essentially from unpacking the definitions of regret \eqref{eq:instantaneous-regret} and the representation of the action selection kernel $p_m$ in the algorithm
    and \eqref{eq:connecting-p-and-Q}:
    \begin{align*}
        &\sum_{\pi\in\Psi} Q_m(\pi)\Reg_{\hatg_m}(\pi) \\
        & = \sum_{\pi\in\Psi} Q_m(\pi)\E_{x\sim D_{\Xscript}}\Big[ \hatg_m(x,\pi_{\hatg_m}(x)) - \hatg_m(x,\pi(x)) \Big]\\
        & = \E_{x\sim D_{\Xscript}}\Big[ \sum_{\pi\in\Psi} Q_m(\pi)\Big(\hatg_m(x,\pi_{\hatg_m}(x)) - \hatg_m(x,\pi(x))\Big) \Big]\\
        & = \E_{x\sim D_{\Xscript}}\Big[ \sum_{a\in\A} \sum_{\pi\in\Psi} \I(\pi(x)=a) Q_m(\pi)\Big(\hatg_m(x,\pi_{\hatg_m}(x)) - \hatg_m(x,a)\Big) \Big]\\
        & = \E_{x\sim D_{\Xscript}}\Big[ \sum_{a\in\A} p_m(a|x) \Big(\hatg_m(x,\pi_{\hatg_m}(x)) - \hatg_m(x,a)\Big) \Big]\\
        & = \E_{x\sim D_{\Xscript}}\Bigg[ \sum_{a\in\A} \frac{\Big(\hatg_m(x,\pi_{\hatg_m}(x)) - \hatg_m(x,a)\Big)}{K+\gamma_m\Big(\hatg_m(x,\pi_{\hatg_m}(x)) - \hatg_m(x,a)\Big)} \Bigg] \leq \frac{K}{\gamma_m}.
    \end{align*}
\end{proof}

Lemma \ref{lem:InverseProbBound} states that the expected inverse probability is bounded by the estimated expected instantaneous regret of policy $\pi$.

\begin{restatable}[]{lemma}{lemInverseProbBound}\label{lem:InverseProbBound}
    For any epoch $m$ and any policy $\pi\in\Psi$,
    \begin{equation*}
        V(p_m,\pi)\le K+\gamma_m\Reg_{\hatg_m}(\pi).
    \end{equation*}
\end{restatable}
\begin{proof}
    For any policy $\pi\in\Psi$, given any context $x\in\calX$,
    \begin{equation*}
        \frac{1}{p_m(\pi(x)|x)}\begin{cases}
         =K+\gamma_m\left(\hatg_m(x,\hata_m(x))-\hatg_m(x,\pi(x))\right) & \text{if }\pi(x)\neq\hata_m(x) \\
         \le\frac{1}{1/K}=K=K+\gamma_m\left(\hatg_m(x,\hata_m(x))-\hatg_m(x,\pi(x))\right) & \text{if }\pi(x)=\hata_m(x).
        \end{cases}
    \end{equation*}
    Thus,
    \begin{equation*}
        V(p_m,\pi)=\E_{x\sim D_\calX}\left[\frac{1}{p_m(\pi(x)|x)}\right]\le K+\gamma_m\E_{x\sim D_\calX}\left[\hatg_m(x,\hata_m(x))-\hatg_m(x,\pi(x))\right]=K+\gamma_m\Reg_{\hatg_m}(\pi).
    \end{equation*}
\end{proof}

\subsection{Bounding Treatment Effect Error}
\label{app:bounding-prediction-error}

For any policy, Lemma \ref{lem:reg-est-accuracy} bounds the error of the treatment effect estimate for the first few epochs. Our definition of $\msafe$ and our choice of $\gamma_{m+1}$ allows us to prove this lemma without assuming realizability.

\begin{restatable}{lemma}{lemImpPolEval}\label{lem:reg-est-accuracy}
    Suppose the event $\eventReg$ defined in \eqref{eq:w-event-Reg} holds. Then, for any policy $\pi\in\Psi$ and epoch $m\leq \min\{\msafe,\msafealg\}+1$,
    \begin{align*}
        |\Phi^{\hatg_{m+1}}_{p_m}(\pi)-\Phi_{p_m}(\pi)| \leq \frac{\sqrt{V(p_m,\pi)}\sqrt{K}}{2\gamma_{m+1}}.
    \end{align*}
\end{restatable}
\begin{proof}
    For any policy $\pi$ and epoch $m\leq \min\{\msafe,\msafealg\}+1$,
    \begin{align*}
        &|\Phi^{\hatg_{m+1}}_{p_m}(\pi)-\Phi_{p_m}(\pi)|\\
        \leq & \E_{x\sim D_{\Xscript}}\left[\left|\langle e_{\pi(x)}-p_{m}(x), \hatg_{m+1}(x)-\starg(x)\rangle\right|\right] \\
        = & \E_{x\sim D_{\Xscript}}\left[\sqrt{\frac{1}{p_m(\pi(x)|x)}p_m(\pi(x)|x)\langle e_{\pi(x)}-p_{m}(x), \hatg_{m+1}(x)-\starg(x)\rangle^2}\right] \\
        \leq & \E_{x\sim D_{\Xscript}}\left[\sqrt{\frac{1}{p_m(\pi(x)|x)}\E_{a\sim p_m(\cdot|x)}\left[\langle e_{a}-p_{m}(x), \hatg_{m+1}(x)-\starg(x)\rangle^2 \right]}\right] \\
        \leq & \sqrt{\E_{x\sim D_{\Xscript}}\left[\frac{1}{p_m(\pi(x)|x)} \right]} \sqrt{\E_{x\sim D_{\Xscript}} \E_{a\sim p_m(\cdot|x)}\left[\langle e_{a}-p_{m}(x), \hatg_{m+1}(x)-\starg(x)\rangle^2\right]}\\
        =&\sqrt{\E_{x\sim D_{\Xscript}}\left[\frac{1}{p_m(\pi(x)|x)} \right]} \sqrt{\calE_{p_m}(\hatg_{m+1})} \\
        \leq & \sqrt{V(p_m, \pi)} \sqrt{2\xi\Big(\tau_m-\tau_{m-1}, \frac{\delta'}{m^2} \Big)} \\
        =& \frac{\sqrt{V(p_m,\pi)}\sqrt{K}}{2\gamma_{m+1}}.
    \end{align*}
    The first inequality follows from Jensen's inequality, the second inequality is straight-forward, the third inequality follows from Cauchy-Schwarz inequality, and the last inequality follows from assuming that $\eventReg$ from \eqref{eq:w-event-Reg} holds.
\end{proof}

\subsection{Bounding Implicit Regret Errors and True Regret}
\label{app:bounding-true-regret-early}

The next lemma implies that before misspecification becomes a problem we are able to bound regret. We omit the proof since it is very similar to the proof of Lemma 8 in \citet{simchi2021bypassing}, given how we defined the instantaneous treatment effect and instantaneous regret.
The proof of this result relies on an inductive argument over the epochs and making use of Lemmas \ref{lem:InverseProbBound} and \ref{lem:reg-est-accuracy}.

\begin{restatable}{lemma}{lemboundReg}\label{lem:policy-reg-bound}
    Suppose the event $\eventReg$ defined in \eqref{eq:w-event-Reg} holds. Let $C_1=5.15$. For any policy $\pi\in\Psi$ and epoch $m\leq \min\{\msafe,\msafealg\} + 1$,
    \begin{align*}
        \Reg(\pi) &\leq 2\Reg_{\hatg_{m}}(\pi) + \frac{C_1K}{\gamma_m}\\
        \Reg_{\hatg_{m}}(\pi) &\leq 2\Reg(\pi) + \frac{C_1K}{\gamma_m}
    \end{align*}
\end{restatable}

In other words, for any policy, Lemma \ref{lem:policy-reg-bound} bounds the error of the regret estimate for the first few epochs. Lemma \ref{lem:QmRegTrue} bounds the implicit regret of the randomized policy $Q_m$ for the first few epochs. Lemma \ref{lem:QmRegTrue} and its proof is more or less the same as the statement and the proof of Lemma 9 in \citet{simchi2021bypassing}.

\begin{restatable}{lemma}{lemQmRegTrue}\label{lem:QmRegTrue}
    Suppose the event $\eventReg$ defined in \eqref{eq:w-event-Reg} holds. Then, for any epoch $m\leq \min\{\msafe,\msafealg\} + 1$,
    \begin{equation*}
        \sum_{\pi \in \Psi} Q_m(\pi)\Reg(\pi) \leq \frac{(2+C_1)K}{\gamma_m}.
    \end{equation*}
\end{restatable}
\begin{proof}
    For any $m\leq \min\{\msafe,\msafealg\}+1$:
    \begin{align*}
        \sum_{\pi \in \Psi} Q_m(\pi)\Reg(\pi)
        \leq \sum_{\pi \in \Psi} Q_m(\pi)\bigg( 2\Reg_{\hatg_{m}}(\pi) + \frac{C_1K}{\gamma_m} \bigg)
         \leq \frac{2K}{\gamma_m} + \frac{C_1K}{\gamma_m},
    \end{align*}
    where the first inequality follows from Lemma \ref{lem:policy-reg-bound}, and the second inequality follows from Lemma \ref{lem:QmRegEst}.
\end{proof}

\subsection{Proof of Theorem \ref{thm:main-theorem}}
\label{app:proof-of-main-theorem}

\begin{proof}
    For each round $t\le T$, define
    \begin{equation*}
        M_t\coloneqq r_t(\starpi(x))-r_t(a_t)-\sum_{\pi\in\Psi}Q_{m(t)}(\pi)\Reg(\pi).
    \end{equation*}
    By Lemma \ref{lem:conditional-regret}, we have $\E[M_t|\Gamma_{t-1}]=0$. Since $|M_t|\le 2$, $M_t$ is a martingale difference sequence. By Azuma's inequality,
    \begin{equation*}
        \sum_{t=1}^TM_t\le\sqrt{8T\log(2/\delta)}
    \end{equation*}
    with probability at least $1-\delta/2$. In other words,
    \begin{equation*}
        \sum_{t=1}^T(r_t(\starpi(x_t))-r_t(a_t))\le\sum_{t=1}^T\sum_{\pi\in\Psi}Q_{m(t)}(\pi)\Reg(\pi)+\sqrt{8T\log(2/\delta)}
    \end{equation*}
    with probability at least $1-\delta/2$. We consider two cases, when $T\le\tau_{\msafe+1}$ and $T>\tau_{\msafe+1}$.
    
    \paragraph{Case 1 ($T\le\tau_{\msafe+1}$):}
    If $T\le\tau_{\msafe+1}$, then the status of the algorithm is always safe by Assumption \ref{ass:msafe-bound}. Moreover, by Lemma \ref{lem:high-prob-eventReg}, with probability at least $1-\delta/2$, the event $\eventReg$ holds. Therefore, by a union bound and then using Lemma \ref{lem:QmRegTrue}, we bound
    \begin{align*}
        \sum_{t=1}^T(r_t(\starpi(x_t))-r_t(a_t))&\le\sqrt{8T\log(2/\delta)}+\sum_{t=1}^T\sum_{\pi\in\Psi}Q_{m(t)}(\pi)\Reg(\pi) \\
        &\le\sqrt{8T\log(2/\delta)}+\tau_1+\sum_{t=\tau_1+1}^T\frac{(2+C_1)K}{\gamma_{m(t)}} \\
        &\le\sqrt{8T\log(2/\delta)}+\tau_1+(2+C_1)\sum_{t=\tau_1+1}^T\sqrt{8K\cdot\xi\left(\tau_{m(t)-1}-\tau_{m(t)-1}),\frac{\delta}{2m(t)^2}\right)} \\
        &\le\calO\left(\sqrt{K}\sum_{t=\tau_1+1}^T\sqrt{8\xi\left(\tau_{m(t)-1}-\tau_{m(t)-1}),\frac{\delta}{2m(t)^2}\right)}\right)
    \end{align*}
    with probability at least $1-\delta$.
    
    \paragraph{Case 2 ($T>\tau_{\msafe+1}$):}
    For the cumulative regret up to $\tau_{\msafe+1}$, we can follow the arguments of case 1. For the cumulative regret between $\tau_{\msafe+1}+1$ and $T$, we use Assumption \ref{ass:post-msafe-reg-bound} to bound expected regret in terms of the misspecification error at these rounds. By a union bound, we have
    \begin{align*}
        \sum_{t=1}^T(r_t(\starpi(x_t))-r_t(a_t))&\le\sqrt{8T\log(2/\delta)}+\sum_{t=1}^T\sum_{\pi\in\Psi}Q_{m(t)}(\pi)\Reg(\pi) \\
        &=\sqrt{8T\log(2/\delta)}+\sum_{t=1}^{\tau_{\msafe+1}}\sum_{\pi\in\Psi}Q_{m(t)}(\pi)\Reg(\pi)+\sum_{t=\tau_{\msafe+1}+1}^{T}\sum_{\pi\in\Psi}Q_{m(t)}(\pi)\Reg(\pi) \\
        &\,\begin{aligned}
            \le&\sqrt{8T\log(2/\delta)}+\tau_1+(2+C_1)\sum_{t=\tau_1+1}^{\tau_{\msafe+1}}\sqrt{8K\cdot\xi\left(\tau_{m(t)-1}-\tau_{m(t)-1}),\frac{\delta}{2m(t)^2}\right)} \\
            &+(T-\tau_{\msafe+1}-1)\calO(\sqrt{KB})
        \end{aligned} \\
        &\le\calO\left(\sqrt{K}\sum_{t=\tau_1+1}^T\sqrt{\xi\left(\tau_{m(t)-1}-\tau_{m(t)-1}),\frac{\delta}{2m(t)^2}\right)}+\sqrt{KB}T\right).
    \end{align*}
    with probability at least $1-\delta$.
\end{proof}

\section{UNKNOWN PROPENSITIES}
\label{app:unknown-propensities}

For completeness, consider the case where we also do not know the exact treatment propensities.
It becomes necessary to estimate the treatment propensity nuisance parameter in cases where the treatment assignment probabilities are not logged by the algorithm or when there is treatment non-compliance from the users. In these settings, the results may depend on the estimation of the nuisance parameters.

Let $\hatp$ and $\hatmu$ be the estimates of $p$ and $\mu$, respectively. Consider the approximate $R$-loss risk under kernel $p$ with the estimated nuisance propensities $\hatp$ in the loss function
\begin{align*}
    \calR_{p, \hatp}(g)&\coloneqq\E_{(x,a,r(a))\sim D(p)}\left[\big(r(a)-\hatmu(x)-\langle e_a-\hatp(x),g(x)\rangle\big)^2\right]
\end{align*}
and its corresponding excess risk
\begin{align*}
    \calE_{p,\hatp}(g)&\coloneqq\calR_{p,\hatp}(g)-\min_{g'}\calR_{p,\hatp}(g'),
\end{align*}
where the infimum is over all functions. Note how $\calR_{p,p}(g)=\calR_p(g)$ and $\calE_{p,p}(g)=\calE_p(g)$ from our previous definitions. We aim to establish \textit{true} $R$-loss excess risk bounds in the setting where we also have to estimate the treatment propensities. First, observe that the risk under the estimated kernel $\hatp$ is
\begin{align*}
        \calR_{p,\hatp}(g)&=\E\left[\big(r(a)-\hatmu(x)-\langle e_a-\hatp(x),g(x)\rangle\big)^2\right] \\
        &=\E\left[\big(r(a)-\hatmu(x)-\langle e_a-p(x),g(x)\rangle - \langle p(x)-\hatp(x),g(x)\rangle\big)^2\right] \\
        &\ \begin{aligned}
            =&\E\left[\big(r(a)-\hatmu(x)-\langle e_a-p(x),g(x)\rangle\big)^2\right] \\
            &+\E\left[\langle p(x)-\hatp(x),g(x)\rangle^2-2\big(\mu(x)-\hatmu(x)\big)\langle p(x)-\hatp(x),g(x)\rangle\right]
        \end{aligned} \\
        &=\calR_p(g)+\E\left[\langle p(x)-\hatp(x),g(x)\rangle^2-2\big(\mu(x)-\hatmu(x)\big)\langle p(x)-\hatp(x),g(x)\rangle\right].
\end{align*}
Recall $\calG^*\coloneqq\argmin_{g'}\calR_p(g')$ is the set of true treatment effects models under the $R$-loss risk with the true kernel $p$. For any $\starg\in\calG^*$,
\begin{align*}
   \calR_{p,\hatp}(g)-\calR_{p,\hatp}(\starg)=\calE_p(g)+\E\left[\big(\langle p(x)-\hatp(x),g(x)+\starg(x)\rangle-2(\mu(x)-\hatmu(x))\big)\!\cdot\!\langle p(x)-\hatp(x),g(x)-\starg(x)\rangle\right]
\end{align*}
and, by Cauchy-Schwarz inequality,
\begin{align*}
    |\calR_{p,\hatp}(g)-\calR_{p,\hatp}(\starg)-\calE_p(g)|&=|\E\left[\big(\langle p(x)-\hatp(x),g(x)+\starg(x)\rangle-2(\mu(x)-\hatmu(x))\big)\cdot\langle p(x)-\hatp(x),g(x)-\starg(x)\rangle\right]| \\
    &\le\E\left[\big(\norm{p(x)-\hatp(x)}\norm{g(x)+\starg(x)}+2|\mu(x)-\hatmu(x)|\big)\cdot\norm{p(x)-\hatp(x)}\norm{g(x)-\starg(x)}\right] \\
    &\ \begin{aligned}
        \le\ &2\sqrt{K}\E\left[\norm{g(x)-\starg(x)}^2\right]^{1/2}\cdot\E\left[\norm{p(x)-\hatp(x)}^4\right] \\
        &+ 2\E\left[\norm{g(x)-\starg(x)}^2\right]^{1/2}\cdot\E\left[|\mu(x)-\hatmu(x)|^2\norm{p(x)-\hatp(x)}^2\right]^{1/2} \\
    \end{aligned} \\
     &\ \begin{aligned}
        =\ C_2\E\left[\norm{g(x)-\starg(x)}^2\right]^{1/2}\cdot\Big(&\E\left[\norm{p(x)-\hatp(x)}^4\right]^{1/2} \\
        &+\E\left[|\mu(x)-\hatmu(x)|^2\norm{p(x)-\hatp(x)}^2\right]^{1/2}\Big)
    \end{aligned} \\
    &\coloneqq\Delta(g,\starg),
\end{align*}
where $C_2=2(1+\sqrt{K})\le 4\sqrt{K}$. Since $\calR_{p,\hatp}(\starg)\!=\!\calR_{p,\hatp}(\starf)$, the quantity on the left-hand side of the inequality above does not depend on the choice of $\starg\in\calG^*$ while the bound on the right-hand side does. Therefore, we can take the minimum over $\calG^*$ on the right-hand side of the above inequality to get
\begin{align*}
    |\calR_{p,\hatp}(g)-\calR_{p,\hatp}(\starg)-\calE_p(g)|\le\min_{\starg\in\calG^*}\Delta(g,\starg).
\end{align*}
Now, notice that
\begin{align}
    \calE_p(g)&\le \calR_{p,\hatp}(g)-\calR_{p,\hatp}(\starg)+|\calR_{p,\hatp}(g)-\calR_{p,\hatp}(\starg)-\calE_p(g)|\le \calE_{p,\hatp}(g) + \min_{\starg\in\calG^*}\Delta(g,\starg). \label{eq:unknown-propensity-excess-risk-ineq}
\end{align}

The first term $\calE_{p,\hatp}(g)$ on the right-hand side of the last inequality is the approximate excess risk with the estimated kernel $\hatp$. This can be bounded by the excess risk estimation rates for the $R$-oracle that uses the estimated kernel $\hatp$. Therefore, to bound the excess risk under the true kernel, we need to bound the second term $\min_{\starg\in\calG^*}\Delta(g,\starg)$. This requires us to bound the following:
\begin{align*}
    \min_{\starg\in\G^*}\E[\|g(x)-g^*(x)\|^2]&\leq \min_{g^*\in\G^*}\frac{1}{\min_{x,a} p(a|x)}\E\left[\sum_a p(a|x)(g(x,a)-g^*(x,a))^2\right] \\
    &\leq \frac{1}{\min_{x,a} p(a|x)}\min_{\starg\in\calG^*}\left(\calE_p(g)+\E_{x\sim D_\calX}\left[\E_{a\sim p(\cdot|x)}\left[g(x,a)-\starg(x,a)\right]^2\right]\right),
\end{align*}
where the last inequality follows from the result derived in Appendix \ref{app:misspecification-bound}. Furthermore, note that for a fixed $g\in\calG$, there exists a $\starg\in\calG^*$ such that $\E_{a\sim p(\cdot|x)}[\starg(x,a)]=\E_{a\sim p(\cdot|x)}[g(x,a)]$. In particular, since $\starf\in\calG^*$, the function $\starg(x,a)=\starf(x,a)+\E_{a\sim p(\cdot|x)}[\starf(x,a)-g(x,a)]$ is clearly in $\calG^*$ and it satisfies this condition. Therefore, the second term in the above inequality vanishes when the infimum over $\calG^*$ is taken and we get that
\begin{align}\label{eq:overlap-model-error}
    \min_{\starg\in\G^*}\E[\|g(x)-g^*(x)\|^2]&\leq\frac{\calE_p(g)}{\min_{x,a} p(a|x)}\le\frac{\calE_p(g)}{\eta}
\end{align}
where $\eta\coloneqq\min_{x,a}p(a|x)$ is the minimum overlap of kernel $p$. Now, suppose we can estimate the nuisance parameters with $n$ data samples such that
\begin{gather*}
    \E[\norm{p(x)-\hatp(x)}^4]\le K\theta_n \\
    \E[|\mu(x)-\hatmu(x)|^2\norm{p(x)-\hatp(x)}^2]\le K\theta_n
\end{gather*}
for some rate $\theta_n=\calO(n^{-\kappa})$ with $\kappa\ge\frac{1}{2}$. Note that
\begin{gather*}
    \E[\norm{p(x)-\hatp(x)}^4]\le K\E[\norm{p(x)-\hatp(x)}^2] \\
    \E[|\mu(x)-\hatmu(x)|^2\norm{p(x)-\hatp(x)}^2]\le K\E[|\mu(x)-\hatmu(x)|^2].
\end{gather*}
Therefore, we can estimate the nuisance parameters using standard least squares estimation procedures to guarantee fast estimation rates, i.e., $\kappa=1$, for the above quantities. By combining result \eqref{eq:overlap-model-error} and the above estimation rate assumptions on the nuisance parameters into inequality \eqref{eq:unknown-propensity-excess-risk-ineq}, we get
\begin{align*}
    \calE_p(\hatg)\le\calE_{p,\hatp}(\hatg)+\sqrt{\frac{C_3K\theta_n\calE_p(\hatg)}{\eta}},
\end{align*}
where $C_3=64$. To form bounds on $\calE_p(\hatg)$, we consider the following two cases.

\textbf{Case 1:}
\begin{align*}
    \sqrt{\frac{C_3K\theta_n}{\eta}}\le\frac{\sqrt{\calE_{p}(\hatg)}}{2}
\end{align*}
In this case, we have
\begin{align*}
    \calE_p(\hatg)\le 2\calE_{p,\hatp}(\hatg)\le\calO(\rho_n),
\end{align*}
where we let $\rho_n$ be the rate on the approximate $R$-loss excess risk with the estimated kernel. Similar results discussed in Appendix \ref{app:r-learner-estimation-rate} can be used to establish fast rates for this excess risk. Therefore, the excess risk rate in this domain is controlled by the fast rates on the approximate $R$-loss excess risk with the estimated kernel. This is independent of the quality of nuisance parameter estimation. In this regime, the bandit algorithm can have an aggressive exploration rate.

\textbf{Case 2:}
\begin{align*}
    \frac{\sqrt{\calE_p(\hatg)}}{2}<\sqrt{\frac{C_3K\theta_n}{\eta}}
\end{align*}
In this case,
\begin{align*}
    \calE_p(\hatg)<\frac{4C_3K\theta_n}{\eta}.
\end{align*}

In this domain, the excess risk rate is controlled by the orthogonal nuisance parameter estimation rates. Moreover, in this setting, overlap comes into play. In our bandit setting, the overlap parameter $\eta$ is vanishing as well. The rate of decrease of kernel overlap is dictated by the excess risk estimation rate. In particular, using the excess risk rate in this case gives $\eta=\calO(1/\gamma_n)=\calO(\sqrt{\theta_n/\eta})$. This requires $\eta=\calO(\theta_n^{1/3})$ which implies
\begin{align*}
    \calE_p(\hatg)\le\calO(K\theta_n^{2/3}).    
\end{align*}

Therefore, we find that overlap helps with orthogonality. However, it diminishes the required rate of exploration in the algorithm since the nuisance parameters also have fast rate guarantees. In this regime, the bandit algorithm is not as aggressive in its exploration. If we wanted to adapt the algorithm to these different regimes, we would need to characterize when the different estimation rates dominate kick in.

\section{EXPERIMENTAL SETUP}
\label{app:experimental-setup}

We used the EconML package \citep{econml} for the $R$-learner implementation on a linear model class and we used the VowpalWabbit Coba package \citep{coba} for running our synthetic contextual bandit benchmark simulations. Refer to the Github repositories for package licenses. A single 2.8 GHz Quad-Core Intel Core i7 CPU was used to run the synthetic experiments. The code is available at \href{https://github.com/agcarranza/hte-bandits}{https://github.com/agcarranza/hte-bandits}.

\vfill

\end{document}